\documentclass[3p]{elsarticle}

\usepackage{lipsum}
\makeatletter
\def\ps@pprintTitle{%
 \let\@oddhead\@empty
 \let\@evenhead\@empty
 \def\@oddfoot{}%
 \let\@evenfoot\@oddfoot}
\makeatother

\usepackage{Definitions}
\usepackage{subcaption}
\usepackage{mathtools}
\usepackage{bm}
\usepackage{amsmath}
\usepackage{amssymb}
\usepackage{graphicx}
\usepackage{slashbox}
\usepackage{makecell}
\usepackage{algpseudocode}
\usepackage{algorithm}
\usepackage{setspace}

\newcommand{\D}{\mathcal{D}}
\newcommand{\E}{\mathbb{E}}
\newcommand{\R}{\mathbb{R}}
\newcommand{\linearadv}{\textit{Interpolated Adversarial Training }}
\newcommand{\Req}{\textbf{Require:}\hspace*{0.5em}}
\newcommand{\Xhh}{\hspace*{3mm}}
\newcommand{\XXhh}{\Xhh\Xhh}

\newcommand{\tbhh}{\hspace*{55mm}}

\def\tx{\tilde{x}}
\def\ty{\tilde{y}}

\def\tD{\tilde{\cD}}

\newcommand{\bE}{\mathbb{E}}

\renewcommand{\cal}{\mathcal}

\newcommand{\cD}{{\cal D}}

\DeclareMathOperator{\mix}{{mix}}
\newcommand{\hx}{{\hat x}}

\newcommand{\hdelta}{{\hat \delta}}
\newcommand{\cx}{{\check x}}

\newcommand{\tlambda}{{\tilde \lambda}}

\makeatletter
\newcommand{\printfnsymbol}[1]{%
  \textsuperscript{\@fnsymbol{#1}}%
}
\makeatother

\usepackage{todonotes}

\usepackage{lineno,hyperref}

\hypersetup{
    colorlinks=true
}

\modulolinenumbers[5]

\biboptions{authoryear}

\usepackage{color}
\usepackage[normalem]{ulem}

\begin{document}

\begin{frontmatter}

\title{Interpolated Adversarial Training: Achieving Robust Neural Networks without Sacrificing Too Much Accuracy}



\author{Alex Lamb$^{a,*}$, Vikas Verma$^{a,b,*}$, Kenji Kawaguchi$^c$, Juho Kannala$^b$, Yoshua Bengio$^a$}
\cortext[mycorrespondingauthor]{Equal contribution \\ \textit{Email addresses:} \texttt{lambalex@iro.umontreal.ca} (Alex Lamb), \texttt{vikas.verma@aalto.fi} (Vikas Verma), \texttt{kkawaguchi@fas.harvard.edu} (Kenji Kawaguchi), \texttt{juho.kaanala@aalto.fi} (Juho Kannala), \texttt{yoshua.bengio@mila.quebec} (Yoshua Bengio)}

\address{$^a$Montreal Institute for Learning Algorithms (MILA), Canada
\linebreak $^b$Aalto University, Finland
\linebreak $^c$Harvard University, USA \vspace{-15pt}}

\begin{abstract}
Adversarial robustness has become a central goal in deep learning, both in the theory and the practice.  However, successful methods to improve the adversarial robustness (such as adversarial training) greatly hurt generalization performance on the unperturbed data. 
This could have a major impact on how  the adversarial robustness affects real world systems (i.e. many may opt to forego robustness if it can improve accuracy on the unperturbed data).  We propose Interpolated Adversarial Training, which employs recently proposed interpolation based training methods in the framework of adversarial training.  On CIFAR-10, adversarial training increases the standard test error (when there is no adversary) from 4.43\% to 12.32\%, whereas with our Interpolated adversarial training we retain the adversarial robustness while achieving a standard test error of only 6.45\%.  With our technique, the relative increase in the standard error for the robust model is reduced from 178.1\% to just 45.5\%. Moreover, we provide mathematical analysis of  Interpolated Adversarial Training to confirm its efficiencies  and demonstrate its advantages in terms of robustness and generalization.  
\end{abstract}

\begin{keyword}
Adversarial Robustness, Mixup, Manifold Mixup, Standard Test Error
\end{keyword}

\end{frontmatter}



\section{Introduction}

Deep neural networks have been highly successful across a variety of tasks. This success has driven applications in the areas where reliability and security are critical, including  face recognition~\citep{sharif2017adversarial}, self-driving cars~\citep{bojarski2016end}, health care, and malware detection~\citep{lecun2015deep}.  Security concerns emerge when adversaries of the system stand to benefit from a system performing poorly.  Work on \emph{Adversarial examples}~\citep{szegedy2013adv} has shown that neural networks are vulnerable to the attacks perturbing the data in imperceptible ways.  Many defenses have been proposed, but most of them rely on \textit{obfuscated gradients} \citep{athalye2018obfuscate} to give a false illusion of defense by lowering the quality of the gradient signal, without actually improving robustness \citep{athalye2018obfuscate}.  Of these defenses, only adversarial training \citep{kurakin2016advml} was still effective after addressing the problem of obfuscated gradients.  

However, adversarial training has a major disadvantage: it drastically reduces the generalization performance of the networks on unperturbed data samples, especially for small networks. For example, \cite{madry2017adv} reports that adding adversarial training to a specific model increases the standard test error from 6.3\% to 21.6\% on  CIFAR-10.  This phenomenon makes adversarial training difficult to use in practice.  If the tension between the performance and the security turns out to be irreconcilable, then many systems would either need to perform poorly or accept vulnerability, a situation leading to great negative impact.  


\textbf{Our contribution:} We propose to augment the adversarial training  with the interpolation based training, as a solution to the above problem. 
\begin{itemize}
    \item We demonstrate that our approach substantially improves standard test error while still achieving adversarial robustness, using benchmark datasets (CIFAR10 and SHVN) and benchmark architectures (Wide-ResNet and ResNet): Section \ref{sec:robustexp}. 
    \item We demonstrate that our approach does not suffer from \textit{obfuscated gradient} problem by performing black-box attacks on the models trained with our approach: Section \ref{sec:blackbox}.
    \item We perform PGD attack of higher number of steps (upto 1000 steps) and higher value of maximum allowed perturbation/distortion $epsilon$, to demonstrate that the adversarial robustness of our approach remains at the same level as that of the adversarial training : Section \ref{sec:vary_iter}.
    \item We demonstrate that the networks trained with our approach have lower complexity, hence resulting in improved standard test error : Section \ref{sec:complexity}.

\item 
We mathematically analyze the benefit of the proposed method in terms of robustness and generalization. For robustness, we show that  Interpolated Adversarial Training corresponds to approximately minimizing an upper bound of the adversarial loss with additional adversarial perturbations. This explains why models obtained by the  proposed method preserves the adversarial robustness and can sometimes further improve the robustness when compared to standard adversarial training. For generalization, we prove a new generalization bound for   Interpolated Adversarial Training and analyze the benefits of the proposed method. 
\end{itemize}

\section{Related Work}

\textbf{The trade-off between standard test error and adversarial robustness} has been studied in \citep{madry2017adv,tsipras2018odds,aditi,zhang19p}.
While \cite{madry2017adv,tsipras2018odds,zhang19p} empirically demonstrate this trade-off, \cite{tsipras2018odds,zhang19p}  demonstrate this trade-off theoretically as well on the constructed learning problems. Furthermore, \cite{aditi} study this trade-off from the point-of-view of the statistical properties of the \textit{robust objective} \citep{robopt} and the dynamics of optimizing a robust objective on a neural network, and suggest that  adversarial training requires more data to obtain a lower standard test error. Our results on SVHN and CIFAR-10 (Section~\ref{sec:robustexp}) also consistently show higher standard test error with PGD adversarial training.  

While  \cite{tsipras2018odds} presented  data dependent proofs showing that on certain artificially constructed distributions - it is impossible for a robust classifier to generalize as good as a non-robust classifier.  How this relates to our results is an intriguing question.  Our results suggest that the generalization gap between adversarial training and non-robust models can be substantially reduced through better algorithms, but it remains possible that closing this gap entirely on some datasets is impossible.  An important question for future work is how much this generalization gap can be explained in terms of inherent data properties and how much this gap can be addressed through better models.  

Neural Architecture Search \citep{zoph2016nas} was used to find architectures which achieve high robustness to PGD attacks as well as better test error on the unperturbed data \citep{dogus2018intriguing}.  This improved test error on the unperturbed data and a direct comparison to our method is in Table ~\ref{table:cifar10-white-box-wrn}.  However, the method of \cite{dogus2018intriguing} is computationally very expensive as each experiment requires training thousands of models to search for optimal architectures (9360 child models each trained for 10 epochs in \citealp{dogus2018intriguing}), whereas our method involves no significant additional computation.

In our work we primarily concern ourselves with adversarial training, but techniques in the research area of the provable defenses have also shown a trade-off between robustness and generalization on unperturbed data.  For example, the dual network defense of \cite{kolter2017provable} reported 20.38\% standard test error on SVHN for their provably robust convolutional network (most non-robust models are well under 5\% test error on SVHN). \cite{wong2018provable} reported a best standard test \textit{accuracy} of 29.23\% using a convolutional \textit{ResNet} on CIFAR-10 (most non-robust ResNets have accuracy of well over 90\%).  Our goal here is not to criticize this work, as developing provable defenses is a challenging and important area of work, but rather to show that this problem that we explore with \linearadv (on adversarial training type defenses of \citealp{madry2017adv}) is just as severe with provable defenses, and understanding if the insights developed here carry over to provable defenses, could be an interesting area for future work.  


\textbf{Adversarially robust generalization:} Another line of research concerns \textit{adversarially robust generalization}: the performance of adversarially trained networks on adversarial test examples.
\cite{schmidt2018advrobust} observe that a higher sample complexity is needed for better adversarially robust generalization.
\cite{yin2018rademacher} demonstrate that adversarial training results in higher complexity models and hence poorer adversarially robust generalization. Furthermore, \cite{schmidt2018advrobust} suggest that adversarially robust generalization requires more data and \cite{zhai,yair} demonstrate that unlabeled data can be used to improve adversarially robust generalization. In contrast to their work, in this work we focus on improving the generalization performance on unperturbed samples (standard test error), while maintaining  robustness on unseen adversarial examples at the same level. 

\textbf{Interpolation based training techniques:} Yet another line of research \citep{zhang2017mixup,verma2018manifold,ict,mixmatch, jeong2020interpolationbased, DBLP:journals/corr/abs-1909-11715} shows that simple interpolation based training techniques are able to substantially decrease standard test error in fully-supervised and semi-supervised learning paradigms. Along these lines, \cite{vcmtheory} studies the theoretical properties of interpolation based training techniques such as Mixup \citep{zhang2017mixup}

\section{Background}
\label{sec:background}

\subsection{The Empirical Risk Minimization Framework}

Let us consider a general classification task with an underlying data distribution $\mathcal{D}$ 
which consists of examples $x \in \textbf{X}$ and corresponding labels $y \in \textbf{Y}$. The task is to learn a function $f:\textbf{X}\rightarrow \textbf{Y} $ such that for a given $x$, $f$ outputs corresponding $y$. It can be done by minimizing the  risk $\E_{(x, y) \sim \D}[\mathcal{L}(x, y, \theta)]$, where $\mathcal{L}(\theta, x, y)$ is a suitable loss function for instance the cross-entropy loss and $\theta \in \R^p$ is the set of parameters of function $f$. Since this expectation cannot be computed, therefore a common approach is to to minimize the empirical risk $1/N \sum_{i=1}^N \mathcal{L}(x_i, y_i, \theta)$ taking into account only a finite number of examples drawn from the data distribution $\mathcal{D}$, namely $(x_1, y_1),......,(x_N, y_N)$ .

\subsection{Adversarial Attacks and Robustness}

While the empirical risk minimization framework has been very successful and often leads to excellent generalization on the unperturbed test examples, it has the significant limitation that it doesn't guarantee good performance on examples which are carefully perturbed to fool the model \citep{szegedy2013adv,goodfellow2014adv}. That is, the empirical risk minimization framework suffers from a lack of robustness to adversarial attacks.

\cite{madry2017adv} proposed an optimization view of adversarial robustness, in which the adversarial robustness of a model is defined as a min-max problem. Using this view, the parameters $\theta$ of a function $f$ are learned by minimizing $\rho (\theta)$ as described in Equation \ref{eq:minmax}. $\mathcal{S}$ defines a region of points around each example, which is typically selected so that it only contains visually imperceptible perturbations.


\begin{align}
\label{eq:minmax}
        \min_\theta \rho(\theta)&,\quad \text{ where }\quad 
    \rho(\theta) = \mathbb{E}_{(x,y)\sim\mathcal{D}}\left[\max_{\delta\in 
    \mathcal{S}}
    \mathcal{L}(\theta,x+\delta,y)\right] 
\end{align}

Adversarial attacks can be broadly categorized into two categories: Single-step attacks and Multi-step attacks. We evaluated the performance of our model as a defense against the most popular and well-studied adversarial attack from each of these categories. Firstly, we consider the Fast Gradient Sign Method \citep{goodfellow2014adv} which is a single step and can still be effective against many networks.  Secondly, we consider the projected gradient descent attack \citep{kurakin2016advml} which is a multi-step attack. It is slower than FGSM as it requires many iterations, but has been shown to be a much stronger attack \citep{madry2017adv}.  We briefly describe these two attacks as follows:

\textbf{Fast Gradient Sign Method (FGSM)}.  The Fast Gradient Sign Method ~\citep{goodfellow2014adv} produces
$\ell_\infty$ bounded adversaries by the following the sign of the gradient based perturbation.  This attack is cheap since it only relies on computing the gradient once and is often an effective attack against deep networks \citep{madry2017adv,goodfellow2014adv}, especially when no adversarial defenses are employed.  
\begin{align}
\widetilde{x} = x + \varepsilon \operatorname{sgn}(\nabla_x \mathcal{L}(\theta,x,y)).
\end{align}

\textbf{Projected Gradient Descent (PGD)}.  The projected gradient descent attack ~\citep{madry2017adv}, sometimes referred to as FGSM$^k$, is a multi-step extension of the FGSM attack characterized as follows:
\begin{align}
x^{t+1} = \Pi_{x+ \mathcal{S}} \left( x^t +
\alpha\operatorname{sgn}(\nabla_x \mathcal{L}(\theta,x,y))\right).
\end{align}
initialized with $x^0$ as the clean input $x$. $\mathcal{S}$ formalizes the manipulative power of the adversary. $\Pi$ refers to the projection operator, which in this context means projecting the adversarial example back onto the region within an $\mathcal{S}$ radius of the original data point, after each step of size $\alpha$ in the adversarial attack.

\subsection{Gradient Obfuscation by Adversarial Defenses}
Many approaches have been proposed as a defense against adversarial attacks.
A significant challenge with evaluating defenses against adversarial attacks is that many attacks rely upon a network's gradient.  The defense methods which reduce the quality of this gradient, either by making it flatter or noisier can lead to methods which lower the effectiveness of gradient-based attacks, but which are not actually robust to adversarial examples \citep{athalye2017robust,papernot2016security}.  This process, which has been referred to as gradient masking or gradient obfuscation, must be analyzed when studying the strength of an adversarial defense.  

One method for examining the extent to which an adversarial defense gives deceptively good results as a result of gradient obfuscation relies on the observation that black-box attacks are a strict subset of white-box attacks, so white-box attacks should always be at least as strong as black-box attacks.  If a method reports much better defense against white-box attacks than the black-box attack, it suggests that the selected white-box attack is underpowered as a result of the gradient obfuscation.  Another test for gradient obfuscation is to run an iterative search, such as projected gradient descent (PGD) with an unlimited range for a large number of iterations.  If such an attack is not completely successful, it indicates that the model's gradients are not an effective method for searching for adversarial images, and that gradient obfuscation is occurring.  We demonstrate successful results with \linearadv on these sanity checks in Section ~\ref{sec:blackbox}.  Still another test is to confirm that iterative attacks with small step sizes always outperform single-step attacks with larger step sizes (such as FGSM).  If this is not the case, it may suggest that the iterative attack becomes stuck in regions where optimization using gradients is poor due to gradient masking.  In all of our experiments for \linearadv, we found that the iterative PGD attacks with smaller step sizes and more iterations were always stronger than the FGSM attacks (which take a single large step) against our models, as shown in Table~\ref{table:cifar10-white-box-wrn},  Table~\ref{table:cifar10-white-box-prn18}, Table~\ref{table:svhn-white-box} and Table ~\ref{table:svhn-white-box-prn18}.  

\subsection{Adversarial Training}
Adversarial training \citep{goodfellow2014adv} encompasses crafting adversarial examples and using them during training to increase robustness against unseen adversarial examples. To scale adversarial training to large datasets and large models, often the adversarial examples are crafted using the fast single step methods such as FGSM. However, adversarial training with fast single step methods remains vulnerable to adversarial attacks from a stronger multi-step attack such as PGD. Thus, in this work, we consider adversarial training with PGD.
\label{ref:adv_train}

\section{Interpolated Adversarial Training}
We propose \linearadv (IAT), which trains on interpolations of adversarial examples along with interpolations of unperturbed examples.  We use the techniques of Mixup \citep{zhang2017mixup} and Manifold Mixup \citep{verma2018manifold} as ways of interpolating examples. Learning is performed in the following four steps when training a network with \linearadv. In the first step, we compute the loss from a unperturbed (non-adversarial) batch (with interpolations based on either Mixup or Manifold Mixup).  In the second step, we generate a batch of adversarial examples using an adversarial attack (such as Projected Gradient Descent (PGD) \citep{madry2017adv} or Fast Gradient Sign Method (FGSM) \citep{goodfellow2014adv}). In the third step, we train against these adversarial examples with the original labels, with interpolations based on either Mixup or Manifold Mixup. In the fourth step, we obtain the average of the loss from the unperturbed batch and the adversarial batch and update the network parameters using this loss. Note that following \citep{kurakinGB16,tsipras2018odds}, we use both the unperturbed and adversarial samples to train the model \linearadv and we use it in our baseline adversarial training models as well.  The detailed algorithm is described in Algorithm Block~\ref{alg:AdvMix}.  

As \linearadv combines adversarial training with either Mixup \citep{zhang2017mixup} or Manifold Mixup \citep{verma2018manifold}, we summarize these supporting methods in more detail.  The Mixup method \citep{zhang2017mixup} consists of drawing a pair of samples from the dataset $(x_i,y_i) \sim p_D$ and $(x_j, y_j) \sim p_D$ and then taking a random linear interpolation in the input space $\tilde{x} = \lambda x_i + (1-\lambda)x_j$.  This $\lambda$ is sampled randomly on each update (typically from a Beta distribution).  Then the network $f_\theta$ is run forward on the interpolated input $\tilde{x}$ and trained using the same linear interpolation of the losses $\mathcal{L} = \lambda L(f_\theta(\tilde{x}),y_i) + (1-\lambda) L(f_\theta(\tilde{x}),y_j)$.  Here $L$ refers to a loss function such as cross entropy.  

The Manifold Mixup method \citep{verma2018manifold} is closely related to Mixup from a computational perspective, except that the layer at which interpolation is performed, is selected randomly on each training update.

Adversarial training consists of generating adversarial examples and training the model to give these points the original label.  For generating these adversarial examples during training, we used the Projected Gradient Descent (PGD) attack, which is also known as iterative FGSM.  This attack consists of repeatedly updating an adversarial perturbation by moving in the direction of the sign of the gradient multiplied by some step size, while projecting back to an $L_\infty$ ball by clipping the perturbation to maximum $\epsilon$.  Both $\epsilon$, the step size to move on each iteration, and the number of iterations are hyperparameters for the attack.

\begin{algorithm}[ht!]
\setstretch{1.2}
\caption{ 
The Interpolated Adversarial Training Algorithm}
\label{alg:AdvMix}
\begin{tabbing}
\Req $f_{\theta}$: Neural Network\\
\Req $Mix$: A way of combining examples (Mixup or Manifold Mixup )\\
\Req $D$: Data samples\\
\Req $N$: Total number of updates\\
\Req $Loss$: A function which runs the neural network with $Mix$ applied\\
\Xhh {\bf for} $k= 1, \ldots, N$ {\bf do} \\
      \XXhh Sample $(x_i,y_i)\sim D$ \quad\Comment{Sample batch} \\
      \XXhh $\mathcal{L}_{c} = Loss(f_{\theta}, Mix, x_i,y_i)$ \quad\Comment{Compute loss on unperturbed data \\
      \tbhh  using Mixup (or Manifold Mixup)} \\
      \XXhh $\tilde{x}_i = attack(x_i, y_i)$ \quad\Comment{Run attack (e.g. PGD as in \citealp{madry2017adv})} \\
      \XXhh $\mathcal{L}_{a} = Loss(f_{\theta}, Mix, \tilde{x}_i,y_i)$ \quad\Comment{Compute adversarial loss on adversarial  \\
      \tbhh samples using Mixup (or Manifold Mixup)} \\
      \XXhh $\mathcal{L}  = (\mathcal{L}_{c} + \mathcal{L}_{a})/2$ \quad\Comment {Combined loss}\\
      \XXhh $ g \gets \nabla_{\theta} \mathcal{L}$ \quad\Comment {Gradients of the combined Loss }\\
      \XXhh $\theta \gets \text{Step}(\theta, g_\theta)$ \quad\Comment{Update parameters using gradients $g$ (e.g. SGD )}\\
\Xhh {\bf end for}
\end{tabbing}
\end{algorithm}

\bigskip

\textbf{Why \linearadv helps to improve the standard test accuracy: }We present two arguments for why \linearadv can improve standard test accuracy:

\textbf{Increasing the training set size:} \cite{aditi} has shown that adversarial training could require more training samples to attain a higher standard test accuracy. Mixup \citep{zhang2017mixup} and Manifold Mixup \citep{verma2018manifold} can be seen as the techniques that increase the effective size of the training set by creating novel training samples. Hence these techniques can be useful in improving standard test accuracy.

\textbf{Information compression:}
\cite{tishby2015info,shwartz2017info} have shown a relationship between compression of information in the features learned by deep networks and generalization.  This relates the degree to which deep networks compress the information in their hidden states to bounds on generalization, with a stronger bound when the deep networks have stronger compression.

To evaluate the effect of adversarial training on compression of the information in the features, we performed an experiment where we take the representations learned after training, and study how well these frozen representations are able to successfully predict fixed random labels.  If the model compresses the representations well, then it will be harder to fit random labels. In particular, we ran a small 2-layer MLP on top of the learned representations to fit random binary labels.  In all cases we trained the model with the random labels for 200 epochs with the same hyperparameters.  For fitting 10000 randomly labeled examples, we achieved accuracy of: 92.08\% (Baseline) and  97.00\% (PGD Adversarial Training): showing that adversarial training made the representations much less compressed. 

Manifold Mixup \citep{verma2018manifold} has shown to learn more compressed features. Hence, employing Manifold Mixup with the adversarial training might mitigate the adverse effect of the adversarial training. Using the same experimental setup as above, we achieved accuracy of : 64.17\% (Manifold Mixup) and 71.00\% (IAT using Manifold Mixup).  
 
These results suggest that adversarial training causes the learned representations to be less compressed which may be the reason for poor standard test accuracy. At the same time, IAT with Manifold Mixup significantly reduces the ability of the model to learn less compressed features, which may potentially improve standard test accuracy.

\begin{figure*}
    \centering
    \includegraphics[width=\linewidth]{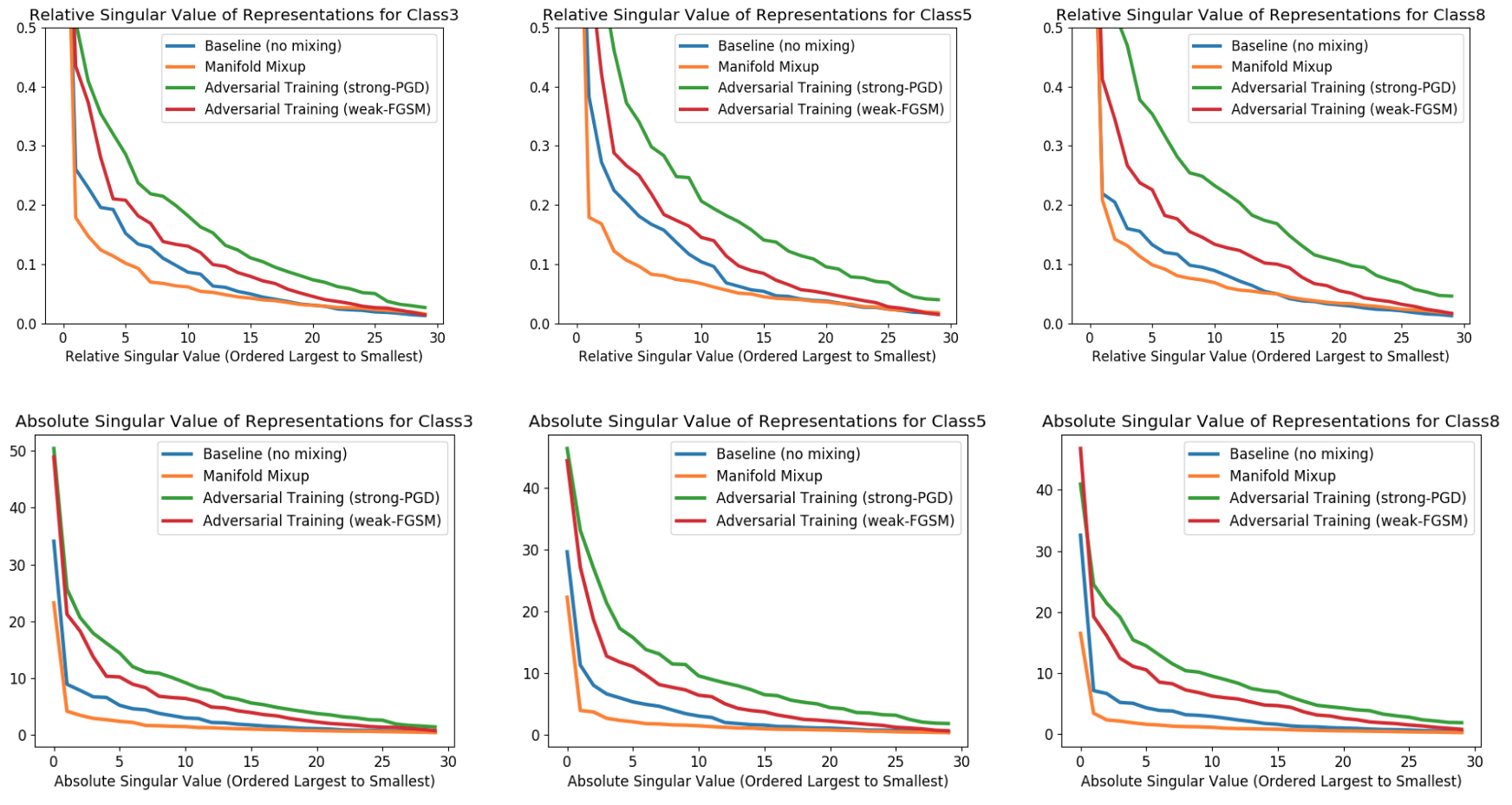}
    \caption{Adversarial Training (especially with PGD) training makes representations have substantially more directions of significant variability (both when measured in an absolute sense and when measured relative to the largest singular value).  }
    \label{fig:svd}
\end{figure*}

\begin{table}[]
\caption{Soft Rank (sum of singular values divided by largest singular value) of the representations (following first layer) from models trained with various methods.  We report separately per MNIST class.  FGSM and PGD refer to models trained with adversarial training.  We note that FGSM slightly increases the numerical rank, but PGD (a much stronger attack) often dramatically increases it.  }
\label{tb:numerical_rank}
\centering
\setlength{\tabcolsep}{0.6em}
{\renewcommand{\arraystretch}{1.2}
\begin{tabular}{|l|l|l|l|l|l}
\cline{1-5}
Class & Baseline & \makecell{Manifold\\Mixup} & FGSM & PGD \\
\cline{1-5}
0 & 2.87 & 2.14 & 3.34 & 3.91 \\\cline{1-5}
1 & 2.90 & 1.91 & 2.92 & 4.15 \\\cline{1-5}
2 & 3.74 & 2.64 & 4.29 & 6.51 \\\cline{1-5}
3 & 3.27 & 2.66 & 4.29 & 5.48 \\\cline{1-5}
4 & 3.18 & 2.41 & 3.58 & 4.70 \\\cline{1-5}
5 & 3.72 & 2.74 & 4.82 & 6.75 \\\cline{1-5}
6 & 3.22 & 2.26 & 3.66 & 5.90 \\\cline{1-5}
7 & 3.43 & 2.39 & 3.66 & 4.42 \\\cline{1-5}
8 & 3.09 & 2.78 & 4.50 & 6.84 \\\cline{1-5}
9 & 3.20 & 2.46 & 3.71 & 5.19 \\\cline{1-5}
\end{tabular}
}
\end{table}

To provide further evidence for a difference in the compression characteristics, we trained 5-layer fully-connected models on MNIST and considered a bottleneck layer of 30 units directly following the first hidden layer.  We then performed singular value decomposition on the per-class representations and looked at the spectrum of singular values (Figure~\ref{fig:svd}).  We found that PGD dramatically increased the number of singular values with large values relative to a baseline model (FGSM was somewhere in-between baseline and PGD).  

\section{Experiments}

\subsection{Adversarial Robustness}
\label{sec:robustexp}

The goal of our experiments is to provide empirical support for our two major assertions: that adversarial training hurts performance on unperturbed data (which is consistent with what has been previously observed in \citealp{madry2017adv,tsipras2018odds,zhang19p}) and to show that this difference can be reduced with our \linearadv method.  Finally, we want to show that \linearadv is adversarially robust and does not suffer from gradient obfuscation \citep{athalye2018obfuscate}.  

In our experiments we always perform adversarial training using a 7-step PGD attack but we evaluate on a variety of attacks: FGSM, PGD (with a varying number of steps and hyperparameters).

\textbf{Architecture and Datasets}: We conducted experiments on competitive networks to demonstrate that \linearadv can improve generalization performance without sacrificing adversarial robustness. We used two architectures : First, the WideResNet architecture proposed in \citep{he2015resnet,zagoruyko2016wrn} and used in \citep{madry2017adv} for adversarial training \footnote{While \cite{madry2017adv} use WRN32-10 architecture, we use the standard WRN28-10 architecture, so our results are not directly comparable to their results.}. Second, the PreActResnet18 architecture which is a variant of the residual architecture of \citep{he2015resnet}. We used SGD with momentum optimizer in our experiments. We ran the experiments for 200 epochs with initial learning rate is 0.1 and it is annealed by a factor of 0.1 at epoch 100 and 150. We use the batch-size of 64 for all the experiments.

We used two benchmark datasets (CIFAR10 and SVHN), which are commonly used in the adversarial robustness literature \citep{madry2017adv}. The CIFAR-10 dataset consists of 60000 color images each of size $32\times32$, split between 50K training and 10K test images.  This dataset has ten classes, which include pictures of cars, horses, airplanes and deer.  The SVHN dataset consists of 73257 training samples and 26032 test samples each of size $32\times32$.  Each example is a close-up image of a house number (the ten classes are the digits from 0-9)

\begin{table*}[ht!]
\setlength{\tabcolsep}{0.1em}
\centering
\setlength{\tabcolsep}{0.2em} 
{\renewcommand{\arraystretch}{1.5}
\caption{CIFAR10 results (error in \%) to white-box attacks on WideResNet20-10 evaluated on the test data. The rows correspond to the training mechanism and columns correspond to adversarial attack methods. The upper part of the Table consists of training mechanisms that do not employ any explicit adversarial defense. The lower part of the Table consist of methods that employ adversarial training as a defense mechanism \protect\footnotemark . For PGD, we used a  $\ell_\infty$ projected gradient descent with size $\alpha = 2$, and $\epsilon = 8$. For FGSM, we used  $\epsilon = 8$. Our method of \linearadv improves standard test error in comparison to adversarial training (refer to the first column) and maintains the adversarial robustness on the same level as that of adversarial training. The method of \cite{dogus2018intriguing} is close to our method in terms of standard test error  and adversarial robustness however it needs several orders of magnitude more computation (it trains 9360 models) for its neural architecture search.}

\label{table:cifar10-white-box-wrn}
\resizebox{\textwidth}{!}{\begin{tabular}{|c|c|c|c|c|}
\hline
\backslashbox{Training}{Adversary}         & No Attack & FGSM   &  \makecell{PGD  (7 steps)} & \makecell{PGD  (20 steps)} \\ \hline
\makecell{Baseline\\ \citep{madry2017adv}}               & 4.80  & 67.3       & 95.9   & 96.5          \\ \hline

\makecell{Baseline}               & \makecell{4.43$\pm$0.09  }  & \makecell{56.92$\pm$0.79 }   & \makecell{99.83$\pm$0.02 }   & \makecell{100.0$\pm$0.0 }  \\ \hline

\makecell{Mixup}              & \makecell{3.25$\pm$0.11 }   & \makecell{ 32.63$\pm$0.88 } & \makecell{92.75$\pm$0.61 }  & \makecell{99.27$\pm$0.03 } \\ \hline 

\makecell{Manifold Mixup}   & 3.15$\pm$0.09   & 38.41$\pm$2.64    & 89.77$\pm$3.68  & 98.34$\pm$1.03 \\ \hline

\makecell {Neural Architecture Search  \citep{dogus2018intriguing}}   & 6.80  & 36.4       & 49.9  & -       \\ \hline
\hline
\makecell {PGD (7 steps) \citep{madry2017adv}}   & 12.70  & 43.90       & 50.00  & 54.20       \\ \hline

\makecell{PGD  (7 steps) (our code)}   & 12.32$\pm$0.14   &  41.87$\pm$0.04    & 50.97$\pm$0.15  & \textbf{54.87$\pm$0.16}        \\ \hline

\makecell{Interpolated  Adversarial Training \\ (with Mixup)}   & \textbf{ 6.45$\pm$0.52}    &  \textbf{33.83$\pm$0.86}       & \textbf{49.88$\pm$0.55}  & 54.89$\pm$1.37        \\ \hline
\makecell{Interpolated  Adversarial Training \\ (Manifold Mixup)}   & 6.48$\pm$0.30  &  35.18$\pm$0.30      & 50.08$\pm$0.48   & 55.18$\pm$0.18    \\ \hline

\end{tabular}}
}
\end{table*}
\footnotetext{Since the objective of this work is to demonstrate the effectiveness the \linearadv over adversarial training for improving the standard test error as well as maintaining the adversarial robustness to the same levels, we highlight the best results in the lower part of the Table: the methods in the upper part of the Table have better standard test error ("No-attack" column), but their adversarial robustness is very poor against strong adversarial attacks (PGD, 7 steps and 20 steps)}

\begin{table*}[ht!]
\setlength{\tabcolsep}{0.1em}
\centering
\setlength{\tabcolsep}{0.2em} 
{\renewcommand{\arraystretch}{1.5}
\caption{CIFAR10 results (error in \%) to white-box attacks on PreActResnet18. Rest of the details are same as Table \ref{table:cifar10-white-box-wrn}}
\label{table:cifar10-white-box-prn18}
\resizebox{\textwidth}{!}{{\begin{tabular}{|c|c|c|c|c|}
\hline
\backslashbox{Training}{Adversary}         & No Attack & FGSM   &  \makecell{PGD  (7 steps)} & \makecell{PGD  (20 steps)} \\ \hline

\makecell{Baseline}   & \makecell{5.88$\pm$0.16 }  & \makecell{78.11$\pm$1.31 }   & \makecell{99.85$\pm$0.18 }   & \makecell{100.0 $\pm$0.0 }  \\ \hline 
\makecell{Mixup}              & \makecell{4.42$\pm$0.03  }   & \makecell{38.32$\pm$0.76 } & \makecell{97.48$\pm$0.15 }  & \makecell{99.88$\pm$0.02 } \\ \hline 
\makecell{Manifold Mixup}   & 4.10$\pm$0.09 &  37.57$\pm$1.31 & 88.50$\pm$3.20  & 97.80$\pm$1.02  \\ \hline 
\hline
\makecell{PGD  (7 steps)}   & 14.12$\pm$0.06   &  48.56$\pm$0.14     & 57.76$\pm$0.19  & \textbf{61.00$\pm$0.24}     \\ \hline 
\makecell{Interpolated  Adversarial Training \\ (with Mixup)}   & \textbf{10.12$\pm$0.33 }  &  \textbf{40.71$\pm$0.65}      & \textbf{55.43$\pm$0.45} & 61.62$\pm$1.01          \\ \hline 
\makecell{ Interpolated  Adversarial Training \\ (Manifold Mixup) }   & 10.30$\pm$0.15   &  42.48$\pm$0.29      &  55.78$\pm$0.67   & 61.80$\pm$0.51  \\ \hline 

\end{tabular}}
}}
\end{table*}

\begin{table*}[htb!]
\setlength{\tabcolsep}{0.1em}
\centering
\setlength{\tabcolsep}{0.2em} 
{\renewcommand{\arraystretch}{1.5}
\caption{SVHN results (error in \%) to white-box attacks on WideResNet20-10 using the 26032 test examples. The rows correspond to the training mechanism and columns correspond to adversarial attack methods. For PGD, we used a  $\ell_\infty$ projected gradient descent with step-size $\alpha = 2$, and $\epsilon = 8$. For FGSM, we used  $\epsilon = 8$. Our method of \linearadv improves standard test error and adversarial robustness.}
\label{table:svhn-white-box}
\resizebox{\textwidth}{!}{{\begin{tabular}{|c|c|c|c|c|}
\hline
\backslashbox{Training}{Adversary}         & No Attack & FGSM   &  \makecell{PGD  (7 steps)} & \makecell{PGD  (20 steps)} \\ \hline
\makecell{Baseline} & \makecell{3.07$\pm$0.03}  & \makecell{39.36$\pm$1.16 } & \makecell{94.00$\pm$0.65 } & \makecell{98.59$\pm$0.13 } \\ \hline
\makecell{ Mixup}              & \makecell{2.59$\pm$0.08} & \makecell{26.93$\pm$1.96 } &\makecell{90.18$\pm$3.43 } & \makecell{98.78$\pm$0.79 }\\ \hline
\makecell{ Manifold Mixup} & \makecell{2.46$\pm$0.01} & \makecell{29.74$\pm$0.99 } &\makecell{77.49$\pm$3.82 } & \makecell{94.77$\pm$1.34 }\\ \hline
\hline
\makecell{PGD (7 steps)}   & \makecell{6.14$\pm$0.13 } & \makecell{29.10$\pm$0.72 } & \makecell{46.97$\pm$0.49 } &  \makecell{53.47$\pm$0.52 } \\ \hline
\makecell{Interpolated  Adversarial Training \\ (with Mixup)}   & \makecell{3.47$\pm$0.11 }  & \makecell{\textbf{22.08$\pm$0.15} }       & \makecell{45.74$\pm$0.11 }  & \makecell{58.40$\pm$0.46 } \\ \hline 
\makecell{Interpolated  Adversarial Training \\ (Manifold Mixup)}   & \makecell{\textbf{3.38$\pm$0.22} }  & \makecell{22.30$\pm$1.07 }      & \makecell{\textbf{42.61$\pm$0.40} } & \makecell{\textbf{52.79$\pm$0.22} } \\ \hline 


\end{tabular}}
}}
\end{table*}

\begin{table*}[ht!]
\setlength{\tabcolsep}{0.1em}
\centering
\setlength{\tabcolsep}{0.2em} 
{\renewcommand{\arraystretch}{1.5}
\caption{SVHN results (error in \%) to white-box attacks on PreActResnet18. Rest of the details are same as Table ~\ref{table:svhn-white-box}.  }
\label{table:svhn-white-box-prn18}
\resizebox{\textwidth}{!}{{\begin{tabular}{|c|c|c|c|c|}
\hline
\backslashbox{Training}{Adversary}         & No Attack & FGSM   &  \makecell{PGD  (7 steps)} & \makecell{PGD  (20 steps)} \\ \hline

\makecell{Baseline}   & \makecell{ 3.47$\pm$0.09}  & \makecell{ 50.73$\pm$0.22 }   & \makecell{ 96.37$\pm$0.12 }   & \makecell{ 98.61$\pm$0.06 }  \\ \hline 
\makecell{Mixup}              & \makecell{ 2.91$\pm$0.06 }   & \makecell{31.91$\pm$0.59 } & \makecell{ 98.43$\pm$0.85 }  & \makecell{99.95$\pm$0.02 } \\ \hline 

\makecell{Manifold Mixup}   & 2.66$\pm$0.02  &  29.86$\pm$3.60   & 72.47$\pm$1.82  & 94.00$\pm$0.96  \\ \hline 
\hline
\makecell{PGD  (7 steps)}   & 5.27$\pm$0.13  &   26.78$\pm$0.62  & 47.00$\pm$0.22  & 54.40$\pm$0.42          \\ \hline 
\makecell{ Interpolated  Adversarial Training \\ (with Mixup)}   & 3.63$\pm$0.05   &  \textbf{23.57$\pm$0.64}     & 47.69$\pm$0.22 & 54.62$\pm$0.18          \\ \hline 
\makecell{ Interpolated  Adversarial Training \\ (Manifold Mixup) }   & \textbf{3.61$\pm$0.22}   &  24.95$\pm$0.92  & \textbf{ 46.62$\pm$0.28}   & \textbf{54.13$\pm$1.08}  \\ \hline

\end{tabular}}
}}
\end{table*}

\textbf{Data Pre-Processing and Hyperparameters:} The data augmentation and pre-processing is exactly the same as in \citep{madry2017adv}. Namely, we use random cropping and horizontal flip for CIFAR10. For SVHN, we use random cropping. We use the per-image standardization for pre-processing.  For adversarial training, we generated the adversarial examples using a PGD adversary using a  $\ell_\infty$ projected gradient descent with 7 steps of size $2$, and $\epsilon = 8$. For the adversarial attack, we used an FGSM adversary with $\epsilon = 8$ and a PGD adversary with 7 steps and 20 steps of size $2$ and $\epsilon = 8$.  

In the \linearadv experiments, for generating the adversarial examples, we used PGD with the same hyper-parameters as described above. For performing interpolation, we  used either Manifold Mixup with $\alpha = 2.0$ as suggested in \citep{verma2018manifold} or Mixup with  
$alpha=1.0$ as suggested in \citep{zhang2017mixup}. For Manifold Mixup, we performed the interpolation at a randomly chosen layer from  the input layer, the output of the first resblock or  the output of the second resblock, as recommended in \citep{verma2018manifold}.  

\textbf{Results:} The results are presented in Table~\ref{table:cifar10-white-box-wrn},  Table~\ref{table:cifar10-white-box-prn18}, Table~\ref{table:svhn-white-box} and Table~\ref{table:svhn-white-box-prn18}. We observe that IAT consistently improves standard test error relative to models using just adversarial training, while maintaining adversarial robustness at the same level. For example, in Table~\ref{table:cifar10-white-box-wrn}, we observe that the baseline model (no adversarial training) has standard test error of $4.43\%$ whereas PGD adversarial increase the standard test error to $12.32\%$: a relative increase of $178\%$ in standard test error. With \linearadv, the standard test error is reduced to $6.45\%$, a relative increase of only $45\%$ in standard test error as compared to the baseline, while the degree of adversarial robustness remains approximately unchanged, across varies type of adversarial attacks.

\subsection{Transfer Attacks}
\label{sec:blackbox}
As a sanity check that \linearadv does not suffer from gradient obfuscation \citep{athalye2018obfuscate}, we performed a transfer attack evaluation on the CIFAR-10 dataset using the PreActResNet18 architecture.  In this type of evaluation, the model which is used to generate the adversarial examples is different from the model used to evaluate the attack.  As these transfer attacks do not use the target model's parameters to compute the adversarial example, they are considered black-box attacks.  In our evaluation (Table~\ref{tb:transfer}) we found that black-box transfer were always substantially weaker than white-box attacks, hence \linearadv does not suffer from gradient obfuscation \citep{athalye2018obfuscate}. Additionally, in Table \ref{tb:epsilon}, we observe that increasing $\epsilon$ results in $100\%$ success of attack, providing added evidence that \linearadv does not suffer from gradient obfuscation \citep{athalye2018obfuscate}. 

\begin{table*}[ht!]
\centering
\setlength{\tabcolsep}{0.3em} 
{\renewcommand{\arraystretch}{1.5}
\caption{Transfer Attack evaluation of \linearadv on CIFAR-10 reported in terms of error rate (\%).  Here we consider three trained models, using normal adversarial training (Adv), IAT with mixup (IAT-M), and IAT with manifold mixup (IAT-MM).  On each experiment, we generate adversarial examples only using the model listed in the column and then evaluate these adversarial examples on the target model listed in the row.  Note that in all of our experiments white box attacks (where the attacking model and target models are the same) led to stronger attacks than black box attacks, which is the evidence that our approach does not suffer from gradient obfuscation \citep{athalye2018obfuscate}.  }
\label{tb:transfer}
\begin{tabular}{|l|l|l|l||l|l|l||l|l|l|}
\hline
\makecell{$\epsilon$} & \multicolumn{3}{l||}{\makecell{2}}    & \multicolumn{3}{l||}{\makecell{5}}  & \multicolumn{3}{l|}{\makecell{10}} \\ \hline
\backslashbox{Target}{Attack} & \makecell{Adv.\\Train}  & \makecell{IAT\\M}        & \makecell{IAT\\MM}    & \makecell{Adv.\\Train}  & \makecell{IAT\\M}  & \makecell{IAT\\MM}   & \makecell{Adv.\\Train}  & \makecell{IAT\\M}  & \makecell{IAT\\MM}  \\ \hline
Adv. Train  & 28.54 & 21.11   &  21.87    & 43.68 & 28.10  & 29.21    & 74.66 & 44.39 & 48.14   \\ \hline
IAT-M   & 17.14   & 25.57 & 18.07   & 25.02    & 45.03 & 28.85   & 48.74  & 78.49 & 51.35    \\ \hline
IAT-MM    & 18.57     & 18.74   & 25.71 & 26.84    & 26.7  & 43.23 & 50.43 & 48.11 & 77.05 \\ \hline
\end{tabular}}
\end{table*}

\subsection{Varying the number of iterations and $\epsilon$ for Iterative Attacks}
\label{sec:vary_iter}

\begin{table*}[ht!]
\centering
\setlength{\tabcolsep}{0.4em} 
{\renewcommand{\arraystretch}{1.5}
\caption{Robustness on CIFAR-10 PreActResNet18 (Error \%) with increasing $\epsilon$ and a fixed number of iterations (20).  \linearadv and adversarial training both have similar degradation in robustness with increasing $\epsilon$, but \linearadv tends to be slightly better for smaller $\epsilon$ and adversarial training is slightly better for larger $\epsilon$}
\label{tb:epsilon}
\begin{tabular}{|l|l|l|l|l|l|l|l|l|l|}
\hline
\backslashbox{Model}{Attack $\epsilon$} & 1 & 2 & 10 & 15 & 20 & 25 & 50  \\\hline
Adversarial Training & 21.44 & 28.54 & 74.66  & 92.43 & 98.53 & 99.77 & 100.0 \\\hline
IAT (Mixup) & 17.90 & 25.57 & 78.49  & 93.73 & 98.54 & 99.72 & 100.0 \\\hline
IAT (Manifold Mixup) & 18.24 & 25.71 & 77.05  & 93.31 & 98.67 & 99.85 & 100.0 \\\hline
\end{tabular}
}
\end{table*}

\begin{table*}[ht!]
\centering
\setlength{\tabcolsep}{0.6em} 
{\renewcommand{\arraystretch}{1.5}
\caption{Robustness on CIFAR-10 PreActResNet-18 (Error \%) with fixed $\epsilon=5$ and a variable number of iterations used for the adversarial attack.  }
\label{tb:steps}
\begin{tabular}{|l|l|l|l|l|l|l|l|}
\hline
\backslashbox{Model}{Num. Iterations} & 5 & 10 & 20 & 50 & 100 & 1000 \\\hline
Adversarial Training & 42.35 & 43.44 & 43.68 & 43.76 & 43.80 & 43.83 \\\hline
IAT (Mixup) & 41.29 & 44.23 & 45.03 & 45.31 & 45.42 & 45.56 \\\hline
IAT (Manifold Mixup) & 40.74 & 42.72 & 43.23 & 43.43 & 43.51 & 43.60 \\\hline
\end{tabular}
}
\end{table*}

To further study the robustness of \linearadv, we studied the effect of changing the number of attack iterations and the range of the adversarial attack $\epsilon$.  Some adversarial defenses \citep{engstrom2018evaluating} have been found to have increasing vulnerability when exposed to attacks with a large number of iterations.  We studied this (Table~\ref{tb:steps}) and found that both adversarial training and \linearadv have robustness which declines only slightly with an increasing number of steps, with almost no difference between the 100 step attack and the 1000 step attack.  Additionally we varied the $\epsilon$ to study if \linearadv was more or less vulnerable to attacks with $\epsilon$ different from what the model was trained on.  We found that \linearadv is somewhat more robust when using smaller $\epsilon$ and slightly less robust when using larger $\epsilon$ (Table~\ref{tb:epsilon}).

\section{Theoretical Analysis}
\label{sec:complexity}
In this section, we establish mathematical properties of  IAT with Mixup.
We begin in Section \ref{sec:1} with additional notation and then analyze the effect of IAT\ on adversarial robustness in Section \ref{sec:2}.  Moreover, we discuss the effects of IAT\ on generalization in Section \ref{sec:3} by showing how ICT\ can reduce overfitting and lead to better   generalization behaviors. The proofs of all theorems and propositions are presented in \ref{app:B} using a key lemma proven in  \ref{app:A}.

\subsection{Notation} \label{sec:1}
In order to present our analysis succinctly, we introduce additional notation as follows. The standard mixup loss $\Lcal_c$  can be written as 
\begin{equation}
\Lcal_c=\frac{1}{n^2}\sum_{i,j=1}^n \EE_{\lambda\sim \Dcal_\lambda}\ell(f_{\theta}(\tx_{i,j}(\lambda)),\ty_{i,j}(\lambda)),
\end{equation}
where
$
\tx_{i,j}(\lambda)=\lambda x_i+(1-\lambda) x_j,
$
$
\ty_{i,j}(\lambda)=\lambda y_i+(1-\lambda)y_j
$, and
$
\lambda\in[0,1]. 
$
Here, $\Dcal_\lambda$ represents the Beta distribution $Beta(\alpha,\beta)$ with some hyper-parameters $\alpha,\beta>0$. Similarly, the adversarial-mixup loss $\Lcal_a$ used in IAT\ can be defined by 
\begin{equation}
\Lcal_a=\frac{1}{n^2}\sum_{i,j=1}^n \EE_{\lambda\sim \Dcal_\lambda}\ell(f_{\theta}(\cx_{i,j}(\lambda)),\ty_{i,j}(\lambda)),
\end{equation}
where 
$
\hdelta_{i}= \argmax_{\delta_i:\|\delta_i\|_{\rho}\le \epsilon} \ell(f_{\theta}(x_i+\delta_i),y_i),
$
$
\hx_i = x_i + \hdelta_i,
$
and
$
\cx_{i,j}(\lambda)=\lambda\hx_i+(1-\lambda) \hx_j. 
$
Using these two types of losses, the whole IAT loss is defined by
\begin{equation}
\Lcal = \frac{\Lcal_c +\Lcal_a}{2}.   
\end{equation}
In  this section, we focus on the  following family of loss functions: 
$
\ell(q,y)=h(q)-yq,
$ for some twice differentiable function
$h$. This family of loss functions $\ell$ includes many commonly used losses, including the logistic loss and the cross-entropy loss. 

Given a set $\Fcal$ of functions $x\mapsto f(x)$, the Radamachor complexity \citep{bartlett2002rademacher,mohri2012foundations} of the set $\ell \circ \Fcal=\{ (x,y)\mapsto\ell(f(x;\theta),y): f\in \Fcal\}$ can be defined by 
$$
\Rcal_{n}(\ell \circ \Fcal):=\EE_{S,\sigma}\left[\sup_{f \in \Fcal}\frac{1}{n} \sum_{i=1}^n \sigma_i \ell(f(x_{i}),y_{i})\right],
$$ 
where $S=((x_i, y_i))_{i=1}^n$. Here, $\sigma_1,\dots,\sigma_n$ are independent uniform random variables taking values in $\{-1,1\}$.   We  denote by $\tD_\lambda$ the uniform mixture of two Beta distributions,  $\frac{\alpha}{\alpha+\beta}Beta(\alpha+1,\beta)+\frac{\beta}{\alpha+\beta}Beta(\beta+1,\alpha)$. We let   $\cD_\hx$ be the empirical distribution of the perturbed training samples $(\hx_1,\cdots,\hx_n)$, and
define   $\cD_x$ to be the empirical distribution of the  training samples $(x_1,\cdots,x_n)$.
Let  $\cos(a, b)$ be the cosine similarity of two vector $a$ and $b$.

\subsection{The effect of IAT on Robustness} \label{sec:2}
In this subsection, we study how  adding Mixup to adversarial training affects the robustness of the model   by analyzing the adversarial-mixup loss $\Lcal_a$ used in IAT. 
This subsection focuses on the binary cross-entropy loss  $\ell$ by setting $h(z)=\log(1+e^z)$ and $y \in \{0,1\}$, whereas the next section  considers a more general setting. We define a set $\Theta$ of parameter vectors by 
$$
\Theta = \left\{\theta \in \RR^{d} :  y_i (f_{\theta}(x_{i}+ \hdelta_{i}))+(y_i-1)(f_{\theta}(x_{i}+ \hdelta_{i}))\ge 0 \text{ for all } i = 1,\dots, n \right\}.
$$
Note that this set $\Theta$ contains the set of all parameter vectors with correct classifications of training points   (before Mixup) as   
$$
\Theta \supseteq \{\theta \in \RR^{d} :\text{ $\one\{f_{\theta}(x_{i}+\hdelta_{i}) \ge 0\}=y_{i}$ for all $i=1,\dots,n$}\}. 
$$    
Therefore, the condition of $\theta \in \Theta$ is satisfied when the deep neural network classifies  all labels correctly for  the training data with  perturbations before Mixup.
 As the training error (although not training loss) becomes  zero in finite time in many practical cases, the condition of $\theta \in \Theta$ is  satisfied in   finite time in many practical cases.
Accordingly, we study the effect of IAT on robustness in the regime of $\theta \in \Theta$.

Theorem \ref{thm:10} shows that the adversarial-mixup loss $\Lcal_a$  is approximately an upper bound of the adversarial loss with the adversarial perturbations of $x_i \mapsto x_i+\hdelta_{i}+\delta^{\mix}_{i}$  where $\|\delta_i\|_{\rho}\le \epsilon $  is the standard adversarial perturbation and   $\|\delta_{i}\|_2 \le  \epsilon^{\mix}_i$ is the    non-standard \textit{additional} perturbation due to IAT. In other words,  IAT is approximately minimizing the upper bound of the adversarial loss with additional adversarial perturbation $\|\delta_{i}\|_2 \le  \epsilon^{\mix}_i$ with data-dependent radius $\epsilon^{\mix}_i$ for each $i \in \{1,\dots,n\}$. Therefore,  adding Mixup to adversarial training (i.e., IAT)   does not decrease the effect of  the original adversarial training on the robustness approximately (where the approximation error is in the order of $(1-\tlambda)^3$ as discussed below). This is non-trivial because, without Theorem \ref{thm:10},  it is uncertain whether or not adding Mixup  reduces the effect of  adversarial training in terms of the robustness. Moreover, Theorem \ref{thm:10} shows that  IAT\ further improves the robustness depending on the values of data-dependent radius $\epsilon^{\mix}_i$ when compared to standard adversarial training without Mixup. These are consistent with our experimental observations.  
\begin{theorem} \label{thm:10}
Let $\theta \in \Theta$ be a point such that $\nabla f_\theta(x_{i}+\hdelta_{i}) $ and $\nabla^{2} f_\theta(x_{i}+\hdelta_{i})$ exist for all $i=1,\dots,n$.  Assume that   $f_{\theta}(x_{i}+\hdelta_{i})=\nabla f_\theta(x_{i}+\hdelta_{i})\T (x_{i}+\hdelta_{i})$ and $\nabla^2 f_\theta(x_i+\hdelta_{i})=0$ for all $i \in \{1,\dots, n\}$. Suppose that $ \bE_{r\sim\cD_\hx}[r] =0$ and  $\| x_i +\hdelta_{i}\|_{2} \ge c_x \sqrt{d}$ for all $i \in \{1,\dots, n\}$.  Then, there exists a pair $(\varphi, \bar \varphi)$ such that $\lim_{z \rightarrow 0}\varphi(z)$, $\lim_{z \rightarrow 0} \bar \varphi(z)=0$, and
\begin{align*}
\Lcal_a \ge  \frac{1}{n}\sum_{i=1}^n \max_{\substack{ \\ \|\delta^{\mix}_{i}\|_2 \le \epsilon^{\mix}_i }} \ell(f_{\theta}(x_{i}+\hdelta_{i}+\delta^{\mix}_{i}), y_i)+E_{1}+E_{2},
\end{align*}
where $\epsilon^{\mix}_i := R_{i}c_x\EE_\lambda[(1- \lambda)]\sqrt{d}$,   $R_{i} :=|\cos(\nabla f_\theta(x_{i}+\hdelta_{i}),  x_{i}+\hdelta_{i})|$,   $E_{1}:=\bE_{\tlambda\sim\tD_\lambda} [(1-\tlambda )^2 \varphi_{}(1-\tlambda)]$, and $E_2:=\bE_{\tlambda\sim\tD_\lambda} [(1- \tlambda)]^{2}\bar \varphi(\bE_{\tlambda\sim\tD_\lambda} [(1- \tlambda)])$. 
\end{theorem}

\begin{figure*}[t!]
    \centering
\begin{subfigure}[b]{0.45\textwidth}\centering
  \includegraphics[width=1.0\linewidth,height=0.6\linewidth]{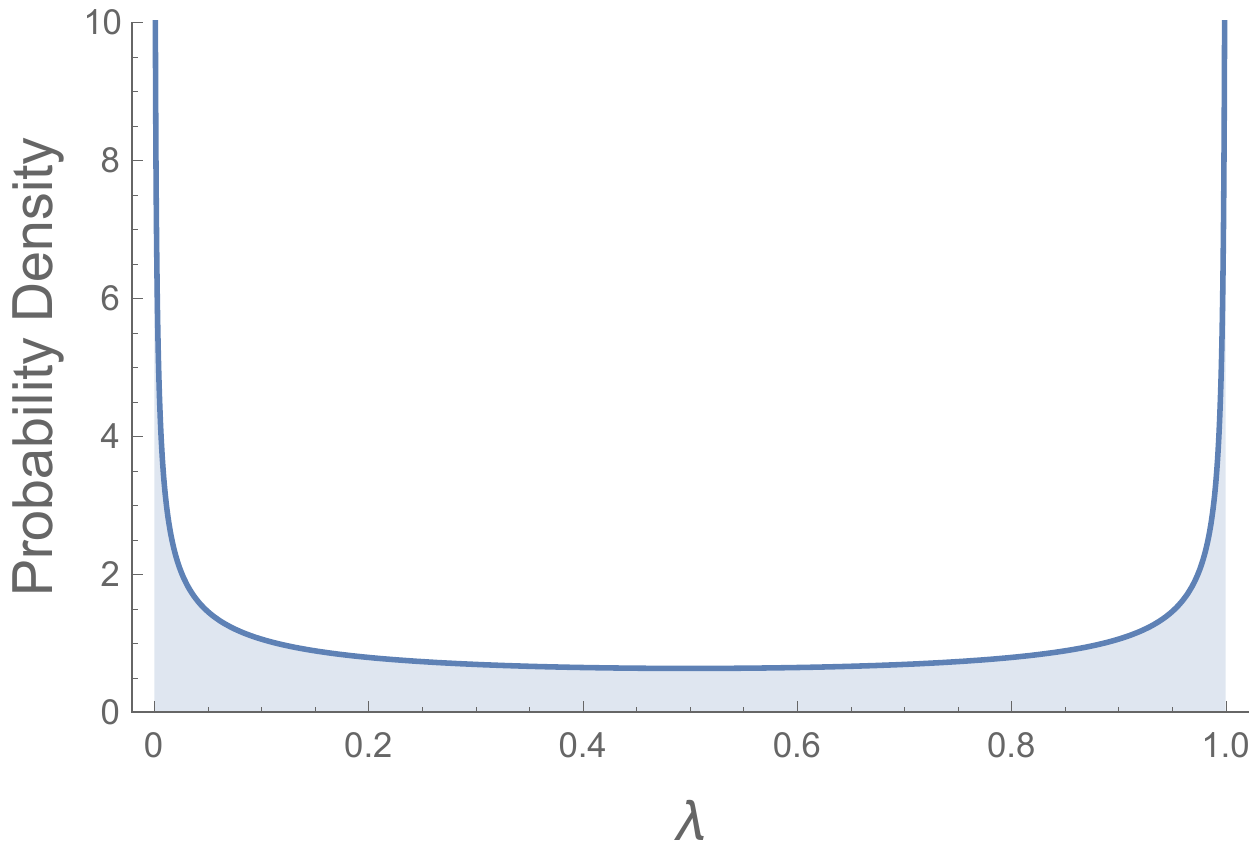}
  \caption{$\D_\lambda=  Beta(0.5, 0.5)$}
\end{subfigure}     
\begin{subfigure}[b]{0.45\textwidth}\centering
  \includegraphics[width=1.0\linewidth,height=0.6\linewidth]{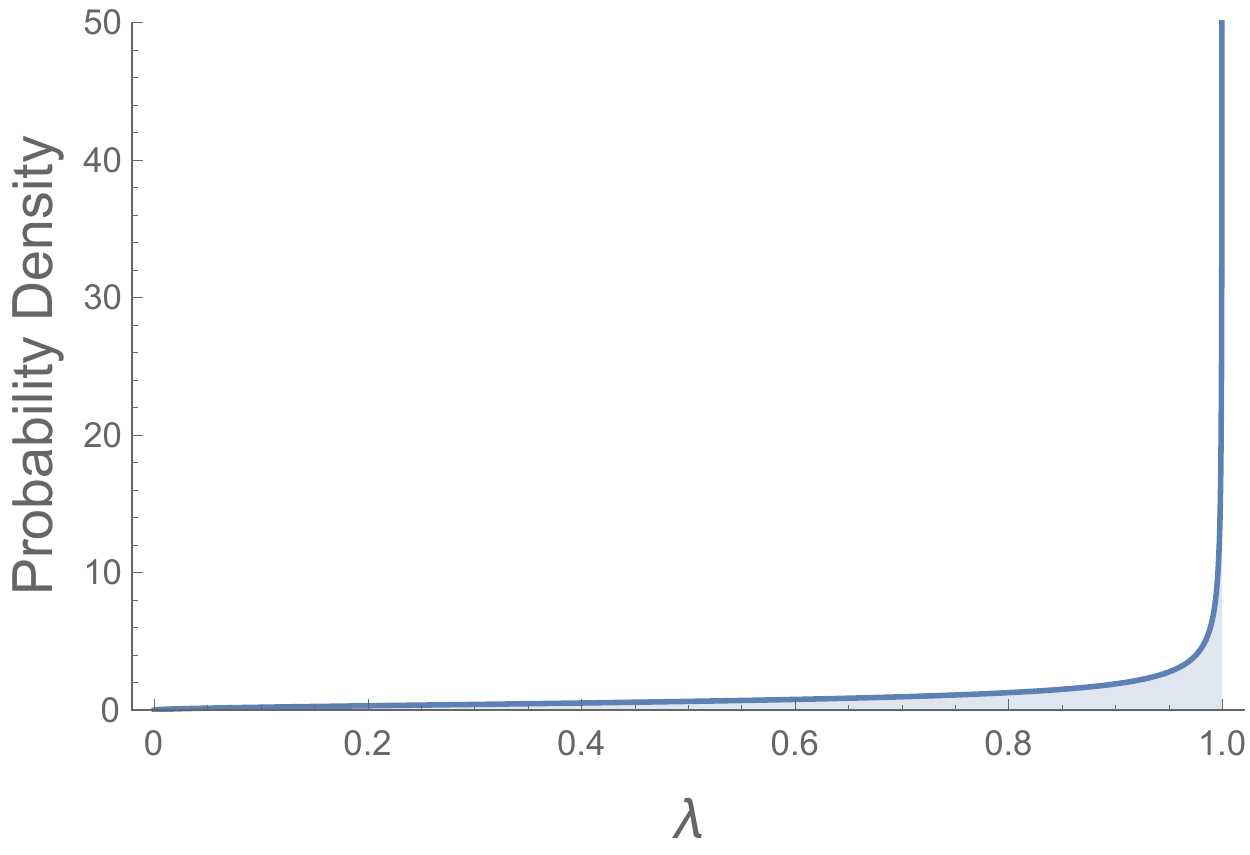}
  \caption{$\tD_\lambda =  Beta(1.5, 0.5)$}
\end{subfigure}     
    \caption{The comparison of the distribution $\D_\lambda$ of the Mixup coefficient $\lambda$ and the distribution $\tD_\lambda$ of $\tlambda$ in the error term $E$. }
    \label{fig:1}
\end{figure*}

The assumption of     $f_{\theta}(z)=\nabla f_\theta(z)\T z$ and $\nabla^2 f_\theta(z)=0$ in Theorem \ref{thm:10} is satisfied, for example, by linear models as well as deep neural networks  with ReLU activation function and max-pooling.
In Theorem \ref{thm:10}, the approximation error terms $E_{1}$ and $E_{2}$ are in the order of $(1-\tlambda)^3$ (since $\lim_{z \rightarrow 0}\varphi(z)$, $\lim_{z \rightarrow 0} \bar \varphi(z)=0$), and $\tlambda$ tends to be close to one since $\tlambda\sim\tD_\lambda$  where  $\tD_\lambda$ is the uniform mixture of two Beta distributions,  $\frac{\alpha}{\alpha+\beta}Beta(\alpha+1,\beta)+\frac{\beta}{\alpha+\beta}Beta(\beta+1,\alpha)$, given the distribution of   $\Dcal_\lambda = Beta(\alpha,\beta)$ for Mixup coefficient $\lambda$. For example, if the IAT algorithm uses the Beta distribution $\Dcal_\lambda =Beta(0.5, 0.5)$ for Mixup, then we have $\tD_\lambda=  Beta(1.5, 0.5)$, for which   $\tlambda \sim \tD_\lambda$ tends to be close to one as illustrated in Figure \ref{fig:1}. Therefore, the approximation error terms $E_{1}$ and $E_{2}$  tend to be close to zero.

\subsection{The effect of IAT on Generalization} \label{sec:3}

In this subsection, we mathematically analyze the  effect of IAT on generalization properties. We start  in Section \ref{sec:4} with th general setting with arbitrary   $h$ and $f_\theta$, and  prove a generalization bound for IAT loss.   In Section \ref{sec:5}, we then make assumptions on    $h$ and $f_\theta$ and    study the regularization effects of IAT.  

\subsubsection{Generalization Bounds} \label{sec:4}
The following theorem presents a generalization bound for the IAT loss $\frac{\Lcal_c +\Lcal_a}{2}$ --- the upper bound on the difference between the expected error on unseen data and the IAT loss, $\EE_{x,y}[\ell(f(x),y)]-\frac{\Lcal_c +\Lcal_a}{2}$:

\begin{theorem} \label{thm:4}
Let $\rho \ge 1$ be a real number and $\Fcal$  be a  set of maps $x\mapsto f(x)$. Assume that the function  $|\ell(q, y)-\ell(q', y)|\le \tau$ for any $q,q' \in\{f(x+\delta):f\in \Fcal, x\in \Xcal,\|\delta\|_{\rho}\le \epsilon\}$ and  $y \in \Ycal$. Then, for any $\delta>0$, with probability at least $1-\delta$ over  an i.i.d. draw of $n$ i.i.d. samples  $((x_i, y_i))_{i=1}^n$, the following holds: for all maps $ f\in\Fcal$,
there exists a function $\varphi: \RR \rightarrow \RR$ such that \begin{align} \label{eq:11}
&\EE_{x,y}[\ell(f(x),y)]-\frac{\Lcal_c +\Lcal_a}{2}
\\ \nonumber &\le  2\Rcal_{n}(\ell \circ \Fcal)+2\tau \sqrt{\frac{\ln(1/\delta)}{2n}} - \frac{Q(f)}{2}-\EE_{X}\left[\sum_{k=1}^3 \frac{G_k(\hat X,\cD_\hx)+G_k(X,\cD_x)}{2}\right]-E_{1},   
\end{align}
where 
 $\lim_{q\rightarrow 0}\varphi(q)=0$, $X:=(x_1,\dots,x_n)$, $\hat X:=(\hx_1,\dots,\hx_n)$,
$$
Q(f):=\frac{1}{n}\EE_{S}\left[\sum_{i=1}^{n}\left(\max_{\delta_i:\|\delta_i\|_{\rho}\le \epsilon}\ell(f(x_i+\delta_i),y_i)-\ell(f(x_i),y_i)\right) \right]\ge 0,
$$  
$$
G_1(\hat X,\cD_\hx):=\frac{\bE_{\lambda\sim\tD_\lambda}[1-\lambda]}{n}\sum_{i=1}^n(h'(f(\hx_i))-y_i)\nabla f(\hx_i)^\top \bE_{r\sim\cD_\hx}[r- \hx_i],
$$
$$
G_2(\hat X,\cD_\hx):=\frac{\bE_{\lambda\sim\tD_\lambda}[(1-\lambda)^2]}{2n}\sum_{i=1}^n h''(f(\hx_i))\nabla f(\hx_i)^\top \bE_{r\sim\cD_\hx}[(r-\hx_i)(r-\hx_i)^\top]\nabla f(\hx_i),
$$
$$
G_3(\hat X,\cD_\hx):=\frac{\bE_{\lambda\sim\tD_\lambda}[(1-\lambda)^2]}{2n}\sum_{i=1}^n (h'(f(\hx_i))-y_i)\bE_{r\sim\cD_\hx}[(r-\hx_i)\nabla^2 f(\hx_i)(r-\hx_i)^\top].
$$
\end{theorem}

To understand this generalization bound further, we now compare it with a generalization bound for IAT without using Mixup on  terms $\Lcal_c$  and $\Lcal_a$. IAT without Mixup is the  adversarial training  along with the standard training, which minimizes the loss of
$$
\Lcal' = \frac{\Lcal_c '+\Lcal_a'}{2},
$$
where 
\begin{equation}
\Lcal_c'=\frac{1}{n}\sum_{i=1}^n \ell(f_{\theta}(x_{i}),y_{i}), \; \text{and }
\end{equation}
\begin{equation}
\Lcal_a'=\frac{1}{n}\sum_{i=1}^n \ell(f_{\theta}(x_i + \hdelta_{i}),y_{i}).
\end{equation}
The following theorem presents a generalization bound for IAT without Mixup on  terms $\Lcal_c$  and $\Lcal_a$:

\begin{theorem} \label{thm:6}
Let $\rho \ge 1$ be a real number and $\Fcal$  be a  set of maps $x\mapsto f(x)$. Assume that the function  $|\ell(q, y)-\ell(q', y)|\le \tau$ for any $q,q' \in\{f(x+\delta):f\in \Fcal, x\in \Xcal,\|\delta\|_{\rho}\le \epsilon\}$ and  $y \in \Ycal$. Then, for any $\delta>0$, with probability at least $1-\delta$ over  an i.i.d. draw of $n$ i.i.d.  samples  $((x_i, y_i))_{i=1}^n$, the following holds: for all maps $ f\in\Fcal$,
\begin{align} \label{eq:1}
\EE_{x,y}[\ell(f(x),y)]-\frac{\Lcal_c '+\Lcal_a'}{2}
 &\le2\Rcal_{n}(\ell \circ \Fcal)+2\tau \sqrt{\frac{\ln(1/\delta)}{2n}} - \frac{Q(f)}{2}.   
\end{align}
 
\end{theorem}

By comparing Theorems \ref{thm:4} and \ref{thm:6}, we can see that the  benefit of IAT with Mixup comes from the two mechanisms in terms of generalization. The first mechanism is based on the term of $\EE_{X}\left[\sum_{k=1}^3 \frac{G_k(\hat X,\cD_\hx)+G_k(X,\cD_x)}{2}\right]+E_{1}$. If this term is positive, then IAT with Mixup has a better generalization bound than that of IAT without Mixup (if we suppose that the Rademacher complexity term $\Rcal_{n}(\ell \circ \Fcal)$ is the same for both methods). The second mechanism is based on the model complexity term $\Rcal_{n}(\ell \circ \Fcal)$. As the model complexity term is bounded by the norms of trained weights (e.g., \citealp{bartlett2017spectral}), this term differs for different training schemes --- IAT with Mixup and IAT without Mixup. Accordingly, we study the regularization effects of IAT on the norms of weights in the next subsection.

\subsubsection{Regularization effects} \label{sec:5}

The generalization bounds in the previous subsection contain the model complexity term, which are controlled by the norms of the weights in the previous studies (e.g., \citealp{bartlett2017spectral}). Accordingly, we now discuss the regularization effects of IAT\ on the norms of weights. This subsection considers the models where    $f_{\theta}(x_{i}+\hdelta_{i})=\nabla f_\theta(x_{i}+\hdelta_{i})\T (x_{i}+\hdelta_{i})$ and $\nabla^2 f_\theta(x_i+\hdelta_{i})=0$ for $i=1,\dots,n$. This is satisfied by linear models as well as deep neural networks with ReLU activation functions and max-pooling. We  let $y \in \{0,1\}$ and $h(z)=\log(1+e^z)$, which makes the loss function $\ell$ to represent the binary cross-entropy loss. Define $g$ to be the logic function as $g(z)= \frac{e^{z}}{1+e^{z}}$. This definition implies that  $g(z) \in (0,1)$ for $z \in \RR$. 

The following theorem shows that the IAT term has the additional regularization effect on $\|\nabla f_\theta(\hx_i)\|_2$ and $\|\nabla f_\theta(\hx_i )\|^2_{\EE_{r}[(r-\hx_i )(r-\hx_i )\T]}$. This theorem explains the additional regularization effects of the IAT term on the norm of weights, since $\nabla f_\theta(\hx_i)=w$ for linear models and $\nabla f_\theta(\hx_i)= \|W^{H} \dot \sigma^{H} W^{H-1} \dot\sigma^{H-1}   \dots  \dot\sigma^{1}W^{1}\|$ for deep neural networks with ReLU and max-pooling. 
 
\begin{theorem}\label{thm:5}
 Assume that   $f_{\theta}(x_{i}+\hdelta_{i})=\nabla f_\theta(x_{i}+\hdelta_{i})\T (x_{i}+\hdelta_{i})$ and $\nabla^2 f_\theta(x_i+\hdelta_{i})=0$. Then, there exists a function $\varphi: \RR \rightarrow \RR$ such that 
\begin{align} 
\Lcal_a
= \frac{1}{n}\sum_{i=1}^n \ell(f_{\theta}(\hx_i ),y_{i})+ C_{1}\|\nabla f_\theta(\hx_i)\|_2+C_{2} \|\nabla f_\theta(\hx_i )\|^2_{\EE_{r}[(r-\hx_i )(r-\hx_i )\T]}+E_{1},
\end{align}
where $\lim_{q\rightarrow 0}\varphi(q)=0$ and 
$$
 C_{1}=\frac{\EE_\lambda [(1-\lambda)]}{n}\sum_{i=1}^n (y_i-g(f_\theta(\hx_i ))) \|\bE_{r\sim\cD_\hx}[r- \hx_i] \|_2 \cos(\nabla f_\theta(\hx_i),\bE_{r\sim\cD_\hx}[r- \hx_i]),
$$
$$
C_{2}=\frac{\EE_\lambda [(1-\lambda)^{2}]}{2n} \sum_{i=1}^n |g(f_\theta(\hx_i ))(1-g(f_\theta(\hx_i )))|.
$$

\end{theorem}

In Theorem \ref{thm:5}, $C_2$ is always strictly positive since $g(z) \in (0,1)$ for all $z\in\RR$. While $C_1$ can be negative in general,  the following proposition shows that $C_1$ will be also non-negative in the later phase of IAT training:

\begin{proposition} \label{prop:1}
If $\theta \in \Theta'$, then $C_{1} \ge 0$ where  
$
\Theta' = \{\theta \in \RR^{d} :  y_i (f_{\theta}(x_{i}+\delta_{i}(\theta))-\zeta_{i})+(y_i-1)(f_{\theta}(x_{i}+\delta_{i}(\theta))-\zeta_{i})\ge 0 \text{ for all } i = 1,\dots, n\},
$     
and 
$
\zeta_{i}=\nabla f_\theta(x_{i}+\delta_{i}(\theta))\T \bE_{r\sim\cD_\hx}[r].
$
\end{proposition}
Here, we have that
$$
\Theta '\supseteq \{\theta \in \RR^{d} :\text{ $\hat \one\{f_{\theta}(x_{i}+\hdelta_{i}) -\zeta_{i}\ge 0\}=$  $y_i$ for all $i=1,\dots,n$}\}. 
$$    
Therefore, the condition of $\theta \in \Theta'$ is satisfied when the model classifies  all labels correctly with margin $\zeta_i$ for  adversarial perturbations.
 As the training error (although not training loss) becomes  zero in finite time in many practical cases and margin increases via implicit bias of gradient descent after that \citep{lyu2020gradient}, the condition of $\theta \in \Theta'$ is  satisfied in   finite time in many practical cases.  Theorem \ref{thm:5} and Proposition \ref{prop:1} together shows that  IAT  can reduce the norms of weights when compared to  adversarial training.  

\citet{zhang2021does} showed that the standard mixup loss $\Lcal_c$ also has the regularization effect on the norm of weights and thus contribute to reduce the model complexity. Therefore, our result together with that of  the previous study \citep{zhang2021does}  shows the  benefit of IAT in terms of reducing the norm of weights to control the model complexity.  As the recent study only considers the standard Mixup without adversarial training, our result complements the recent study to understand IAT.

To validate this theoretical prediction, we   computed the norms of weights for a 6-layer fully-connected network with 512 hidden units trained on Fashion-MNIST and repot the results in Figure~\ref{fig:norms}.  On the one hand, adversarial training increased the  Frobenius norms across all the layers and increased the spectral norm of the majority of the layers.  On the other hand, IAT\ avoided or mitigated these increases in the norms of weights. This is consistent with our theoretical predictions and suggests that IAT  learns lower complexity classifiers than normal adversarial training.  

\begin{figure*}[t!]
    \centering
    \includegraphics[width=0.45\linewidth,height=0.4\linewidth,trim={1.5cm 0cm 1.5cm 0.5cm},clip]{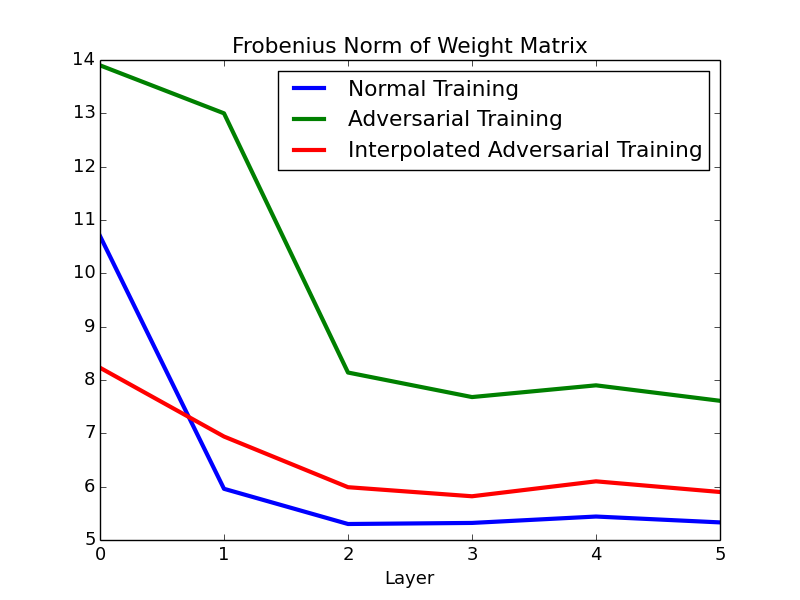}
    \includegraphics[width=0.45\linewidth,height=0.4\linewidth,trim={1.5cm 0cm 1.5cm 0.5cm},clip]{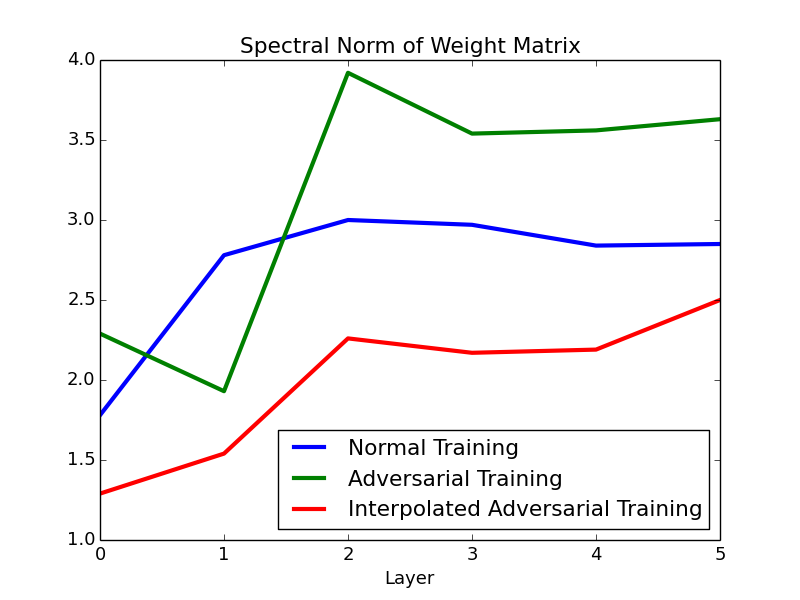}
    \caption{We analyzed the Frobenius and spectral norms of the weight matrices on a 6-layer network.  Generally Adversarial Training makes these norms larger, whereas \linearadv brings these norms closer to their values when doing normal training.  }
    \label{fig:norms}
\end{figure*}

To further understand why  adversarial training tends to increase the norms, consider the case of linear regression:
$$
L(\theta)= \frac{1}{2}\|Xw - Y\|^2_F.
$$ 
Then,  we have
$$
\nabla L(\theta) = X\T (Xw-Y). 
$$
Therefore, each step of  (stochastic) gradient descent only adds some vector in the column space of $X\T$ to $w$ as 
$$
w_{t+1} = w_{t} + v_t \quad \text{ where } v_t \in \Col(X\T). 
$$ 
Here, the solutions of the linear regression are any $w$ such that 
$$
w  = X^\dagger Y +v^\perp \quad \text{ where } v^\perp\in \Null(X).
$$ 
Thus, (stochastic) gradient descent  does not add any unnecessary element to $w$, implicitly minimizing the norm of the weights.  Accordingly, if we initialize $w$ as $w_{0} \in \Col(X\T) $, then we  achieve the minimum norm solution implicitly via (stochastic) gradient descent.  

In this context, we can easily see that by conducting adversarial training, we add vectors   $v^\perp\in \Null(X)$, breaking the implicit bias and increasing the norm of $w$. Similarly, in the case of deep neural networks, (stochastic) gradient descent  has the implicit bias that restrict the search space of $w$ and hence tend to minimize the norm without unnecessary elements \citep{lyu2020gradient,woodworth2020kernel,moroshko2020implicit}.   Thus, similarly to the case of linear models, adversarial training adds extra elements via the perturbation and tend to increases the norm of weights. Our results show that we can minimize this effect via the additional regularization effects of IAT to reduce overfitting for better   generalization behaviors.

\section{Conclusion}
\label{sec:conclusion}
Robustness to the adversarial examples is essential for ensuring that machine learning systems are secure and reliable.  However the most effective defense, adversarial training, has the effect of harming performance on the unperturbed data.  This has both the theoretical and  the practical significance.  As adversarial perturbations are imperceptible (or barely perceptible) to humans and humans are able to generalize extremely well, it is surprising that adversarial training reduces the model's ability to perform well on  unperturbed test data.  This degradation in the generalization is critically urgent to the practitioners whose systems are threatened by the adversarial attacks.  With current techniques those wishing to deploy machine learning systems need to consider a severe trade-off between performance on the unperturbed data and the robustness to the adversarial examples, which may mean that security and reliability will suffer in important applications.  Our work has addressed both of these issues. We proposed to address this by augmenting adversarial training with interpolation based training \citep{zhang2017mixup,verma2018manifold}.  We found that this substantially improves generalization on unperturbed data while preserving adversarial robustness. Our analysis showed why and how the proposed method can  improve  generalization and preserve adversarial robustness when compared to standard adversarial training.

\section*{Acknowledgements} 
The authors thank David Lopez-Paz for useful discussions and feedback.
We would also like to acknowledge Compute Canada for providing computing resources used in this work. The research of Kenji Kawaguchi is partially supported
by the Center of Mathematical Sciences and Applications at Harvard University.

\appendix

\allowdisplaybreaks

\section{Lemma on Adversarial-Mixup Loss} \label{app:A}
This appendix provides the key lemma, Lemmas \ref{lemma:1}, which is used to prove theorems in  \ref{app:B}.
\begin{lemma}\label{lemma:1}
 For any $f_{\theta}$, there exists a function $\varphi: \RR \rightarrow \RR$ such that 
\begin{equation*} 
\Lcal_a=\frac{1}{n}\sum_{i=1}^n \ell(f_{\theta}(\hx_i ),y_{i})+\sum_{i=1}^3 G_i+\bE_{\lambda\sim\tilde\cD_\lambda}[(1-\lambda)^2\varphi(1-\lambda)],
\end{equation*}
where $\lim_{q\rightarrow 0}\varphi(q)=0$ and 
$$
G_1=\frac{\bE_{\lambda\sim\tD_\lambda}[1-\lambda]}{n}\sum_{i=1}^n(h'(f_\theta(\hx_i))-y_i)\nabla f_\theta(\hx_i)^\top \bE_{r\sim\cD_\hx}[r- \hx_i],
$$
$$
G_2=\frac{\bE_{\lambda\sim\tD_\lambda}[(1-\lambda)^2]}{2n}\sum_{i=1}^n h''(f_\theta(\hx_i))\nabla f_\theta(\hx_i)^\top \bE_{r\sim\cD_\hx}[(r-\hx_i)(r-\hx_i)^\top]\nabla f_\theta(\hx_i),
$$
$$
G_3=\frac{\bE_{\lambda\sim\tD_\lambda}[(1-\lambda)^2]}{2n}\sum_{i=1}^n (h'(f_\theta(\hx_i))-y_i)\bE_{r\sim\cD_\hx}[(r-\hx_i)\nabla^2 f_\theta(\hx_i)(r-\hx_i)^\top].
$$
\end{lemma}
\begin{proof}
Since $\ell(q,y)=h(q)-yq$, we have that$$\frac{1}{n}\sum_{i=1}^n \ell(f_{\theta}(\hx_i ),y_{i})=\frac{1}{n}\sum_{i=1}^n[h(f_\theta(\hx_i ))-y_if_\theta(\hx_i )],$$
and 
$$\Lcal_a=\frac{1}{n^2}\bE_{\lambda\sim Beta(\alpha,\beta)}\sum_{i,j=1}^{n}[h(f_\theta(\cx_{i,j}(\lambda)))-(\lambda y_i+(1-\lambda) y_j)f_\theta(\cx_{i,j}(\lambda))].$$
By expanding $(\lambda y_i+(1-\lambda) y_j)f_\theta(\cx_{i,j}(\lambda))$ and using $h(f_\theta(\cx_{i,j}(\lambda)))=\lambda h(f_\theta(\cx_{i,j}(\lambda)))+(1-\lambda)h(f_\theta(\cx_{i,j}(\lambda)))$, \begin{align*}
\Lcal_a&=\frac{1}{n^2}\bE_{\lambda\sim Beta(\alpha,\beta)}\sum_{i,j=1}^{n}\Big\{\lambda[ h(f_\theta(\cx_{i,j}(\lambda)))- y_if_\theta(\cx_{i,j}(\lambda))]
\\ &\hspace{110pt} +(1-\lambda)[h(f_\theta(\cx_{i,j}(\lambda)))- y_jf_\theta(\cx_{i,j}(\lambda))]\Big\}.
\end{align*}
Using the fact that $\EE_{B\sim Bern(\lambda)}[B]=\lambda$,
we have \begin{align*}
&\Lcal_a=\frac{1}{n^2}\bE_{\lambda\sim Beta(\alpha,\beta)}\bE_{B\sim Bern(\lambda)}\sum_{i,j=1}^{n}\Big\{B[h(f_\theta(\cx_{i,j}(\lambda)))- y_if_\theta(\cx_{i,j}(\lambda))]
\\ & \hspace{170pt} +(1-B)[h(f_\theta(\cx_{i,j}(\lambda)))-y_jf_\theta(\cx_{i,j}(\lambda))]\Big\}.
\end{align*}
Since $\lambda\sim Beta(\alpha,\beta), B|\lambda\sim Bern(\lambda)$, by conjugacy, we can exchange them to have 
$$B\sim Bern(\frac{\alpha}{\alpha+\beta}),\lambda\mid B\sim Beta(\alpha+B,\beta+1-B).$$
Thus, \begin{align*}
\Lcal_a&=\frac{1}{n^2}\sum_{i,j=1}^{n}\Big\{\frac{\alpha}{\alpha+\beta}\bE_{\lambda\sim Beta(\alpha+1,\beta)}[h(f_\theta(\cx_{i,j}(\lambda)))- y_if_\theta(\cx_{i,j}(\lambda))]
\\ & \hspace{60pt}+\frac{\beta}{\alpha+\beta}\bE_{\lambda\sim Beta(\alpha,\beta+1)}[h(f_\theta(\cx_{i,j}(\lambda)))-y_jf_\theta(\cx_{i,j}(\lambda))]\Big\}.
\end{align*}
Since $1-Beta(\alpha,\beta+1)$ and $Beta(\beta+1,\alpha)$ represent  the same distribution and $\cx_{ij}(1-\lambda)=\cx_{ji}(\lambda)$, we have
\begin{align*}
&\sum_{i,j}\bE_{\lambda\sim Beta(\alpha,\beta+1)}[h(f_\theta(\cx_{i,j}(\lambda)))-y_jf_\theta(\cx_{i,j}(\lambda))]\\
&=\sum_{i,j}\bE_{\lambda\sim Beta(\beta+1,\alpha)}[h(f_\theta(\cx_{i,j}(\lambda)))-y_if_\theta(\cx_{i,j}(\lambda))].
\end{align*}
By defining $\tD_\lambda=\frac{\alpha}{\alpha+\beta}Beta(\alpha+1,\beta)+\frac{\beta}{\alpha+\beta}Beta(\beta+1,\alpha)$,
\begin{align*} 
\Lcal_a&=\frac{1}{n^{2}}\sum_{i,j=1}^n \bE_{\lambda\sim \tD_\lambda} [h(f_\theta(\cx_{i,j}(\lambda)))-y_if_\theta(\cx_{i,j}(\lambda))]\notag 
\\ &=\frac{1}{n^{2}}\sum_{i,j=1}^n \EE_{\lambda\sim \tD_\lambda}\ell(f_{\theta}(\lambda\hx_i+(1-\lambda) \hx_{j}),y_{i})
\end{align*}
By defining  $\cD_\hx$ to be the empirical distribution induced by perturbed training samples $(\hx_j)_{j=1}^n$, 
\begin{align*} 
\Lcal_a&=\frac{1}{n}\sum_{i=1}^n \EE_{\lambda\sim \tD_\lambda}\EE_{r \sim\cD_\hx}\ell(f_{\theta}(\lambda\hx_i+(1-\lambda)r),y_{i}) 
\end{align*}
Let 
$\check x_i=\lambda\hx_{i}+(1-\lambda) r$,  $\alpha=1-\lambda$, and $\psi_{i}(\alpha)=\ell(f_{\theta}(\cx_i),y_{i})$. Then,  using the definition of the twice-differentiability of function $\psi_{i}$, 
\begin{align*} 
\ell(f_{\theta}(\cx_i),y_{i})=\psi_{i}(\alpha)=\psi_{i}(0)+\psi_{i}'(0)\alpha +\frac{1}{2}  \psi_{i}''(0) \alpha^2 + \alpha^2 \varphi_{i}(\alpha),
\end{align*}
where $\lim_{z \rightarrow 0}\varphi_i(z)=0$. Therefore, 

\begin{align} \label{eq:7}
\Lcal_a&=\frac{1}{n} \EE_{\lambda\sim \tD_\lambda}\EE_{r \sim\cD_\hx}\sum_{i=1}^n [\psi_{i}(0)+\psi_{i}'(0)\alpha +\frac{1}{2}  \psi_{i}''(0) \alpha^2 ]+ \EE_{\lambda\sim \tD_\lambda}[(1-\lambda)^{2}\varphi_{}(1-\lambda)], 
\end{align}
where  $\varphi(\alpha)=\frac{1}{n}\sum_{i=1}^n\varphi_i(\alpha)$.
By linearity and chain rule, \begin{align*}
\psi_{i}'(\alpha) &= h'(f_\theta(\check x_i)) \frac{\partial f_\theta(\check x_i)}{\partial \check x_i} \frac{\partial \check x_i}{\partial \alpha} - y_{i} \frac{\partial f_\theta(\check x_i)}{\partial \check x_i} \frac{\partial \check x_i}{\partial \alpha}
\\ & = h'(f_\theta(\check x_i)) \frac{\partial f_\theta(\check x_i)}{\partial \check x_i} (r - \hx_{i})-y_{i} \frac{\partial f_\theta(\check x_i)}{\partial \check x_i} (r - \hx_{i}) 
\end{align*}
where we used  $\frac{\partial \check x_i}{\partial \alpha}=(r -\hx_{i})$. Moreover,
\begin{align*}
\frac{\partial}{\partial \alpha}\frac{\partial f_\theta(\check x_i)}{\partial \check x_i} (r- \hx_{i})&=\frac{\partial}{\partial \alpha}(r-\hx_{i})\T \left[\frac{\partial f_\theta(\check x_i)}{\partial \check x_i} \right]\T \\ &=(r- \hx_{i})\T\nabla^2 f_\theta(\check x_i)\frac{\partial \check x_i}{\partial \alpha}
\\ & =(r- \hx_{i})\T\nabla^2 f_\theta(\check x_i)(r-\hx_{i}). 
\end{align*} 
Therefore,      
\begin{align*}
\psi_{i}''(\alpha)=&h'(f_\theta(\check x_i)) (r- \hx_{i})\T\nabla^2 f_\theta(\check x_i)(r- \hx_{i})\\
&+h''(f_\theta(\check x_i))[\frac{\partial f_\theta(\check x_i)}{\partial \check x_i} (r- \hx_{i})]^{2}- y_{i} (r- \hx_{i})\T\nabla^2 f_\theta(\check x_i)(r- \hx_{i}).
\end{align*}
By setting $\alpha=0$,
$$
\psi_{i}'(0)=h'(f_\theta( \hx_i)) \nabla f_\theta(\hx_i)\T   (r- \hx_{i})-y_{i} \nabla f_\theta(\hx_i)\T  (r- \hx_{i})=(h'(f_\theta( \hx_i))-y_{i}) \nabla f_\theta(\hx_i)\T   (r- \hx_{i}),
$$
and
\begin{align*}
\psi_{i}''(0) =&h'(f_\theta( \hx_i)) (r- \hx_{i})\T\nabla^2 f_\theta(\hx_i)(r- \hx_{i})
\\ & +h''(f_\theta( \hx_i))[ \nabla f_\theta(\hx_i)\T    (r- \hx_{i})]^{2}
\\ & - y_{i} (r- \hx_{i})\T\nabla^2 f_\theta(\hx_i)(r- \hx_{i})
\\  =&h''(f_\theta( \hx_i)) \nabla f_\theta(\hx_i)\T    (r- \hx_{i})    (r- \hx_{i})\T \nabla f_\theta(\hx_i)
\\ & +(h'(f_\theta( \hx_i))-y_{i}) (r- \hx_{i})\T\nabla^2 f_\theta(\hx_i)(r- \hx_{i}). 
\end{align*}
By substituting these into \eqref{eq:7}, we obtain the statement of this lemma.
\end{proof}

\section{Proofs} \label{app:B}
Using Lemmas \ref{lemma:1} proven in \ref{app:A}, this appendix provides the complete proofs of Theorem \ref{thm:10}, Theorem  \ref{thm:4}, Theorem \ref{thm:6}, Theorem \ref{thm:5}, and Proposition \ref{prop:1}.

\noindent \textbf{Proof of Theorem \ref{thm:10}}. 
Let $\hx_i =x_i + \hat \delta_{i}$. From the assumption, we have $f_{\theta}(\hx_i )=\nabla f_\theta(\hx_i )\T \hx_i $ and $\nabla^2 f_\theta(\hx_i )=0$. Since $h(z)=\log(1+e^z)$, we have $h'(z)=\frac{e^z}{1+e^z}=g(z) \ge 0$ and $h''(z)=\frac{e^z}{(1+e^z)^{2}}=g(z)(1-g(z)) \ge 0$.
By substituting these into the equation of Lemma  \ref{lemma:1} with $ \bE_{r\sim\cD_\hx}[r] =0 $,
\begin{align}\label{eq:14}
\Lcal_a= \frac{1}{n}\sum_{i=1}^n \ell(f_{\theta}(\hx_i ),y_{i})+ G_{1}+ G_{2} +E_{1},
\end{align} 
where 
$$
G_{1}= \frac{\EE_\lambda [(1-\lambda)]}{n}\sum_{i=1}^n (y_i-g(f_\theta(\hx_i )))f_\theta(\hx_i )
$$
\begin{align*}
G_{2}&= \frac{\EE_\lambda [(1-\lambda)^{2}]}{2n} \sum_{i=1}^n |g(f_\theta(\hx_i ))(1-g(f_\theta(\hx_i )))|\nabla f_\theta(\hx_i )\T  \EE_{r}[(r-\hx_i )(r-\hx_i )\T]\nabla f_\theta(\hx_i )
 \\ & \ge\frac{\EE_\lambda [(1-\lambda)]^{2}}{2n} \sum_{i=1}^n |g(f_\theta(\hx_i ))(1-g(f_\theta(\hx_i )))|\nabla f_\theta(\hx_i )\T  \EE_{r}[(r-\hx_i )(r-\hx_i )\T]\nabla f_\theta(\hx_i )  
\end{align*}
where we used  $\EE[z^2]=\EE[z]^2 + \Var(z)\ge E[z]^2$ and $  \nabla f_\theta(\hx_i )\T  \EE_{r}[(r-\hx_i )(r-\hx_i )\T]\nabla f_\theta(\hx_i ) \ge 0$. 
Since $\EE_{r}[(r-\hx_i )(r-\hx_i )\T]=\EE_{r}[rr\T-r\hx_i \T -\hx_i  r\T + \hx_i \hx_i \T]=\EE_{r}[rr\T]+ \hx_i \hx_i \T$ where $\EE_{r}[rr\T]$ is positive semidefinite,
 
\begin{align*}
G_{2}&\ge\frac{\EE_\lambda [(1-\lambda)]^{2}}{2n} \sum_{i=1}^n |g(f_\theta(\hx_i ))(1-g(f_\theta(\hx_i )))|\nabla f_\theta(\hx_i )\T  (\EE_{r}[rr\T]+ \hx_i \hx_i \T)\nabla f_\theta(\hx_i ).
\\ & \ge\frac{\EE_\lambda [(1-\lambda)]^{2}}{2n} \sum_{i=1}^n |g(f_\theta(\hx_i ))(1-g(f_\theta(\hx_i )))|(\nabla f_\theta(\hx_i )\T  \hx_i )^{2} 
\\ & =\frac{\EE_\lambda [(1-\lambda)]^{2}}{2n} \sum_{i=1}^n |g(f_\theta(\hx_i ))(1-g(f_\theta(\hx_i )))| \|\nabla f_\theta(\hx_i )\|_2^2 \|\hx_i \|_2 ^{2} (\cos(\nabla f_\theta(\hx_i ),  \hx_i ))^2 
\\ & \ge \frac{1}{2n} \sum_{i=1}^n |g(f_\theta(\hx_i ))(1-g(f_\theta(\hx_i )))| \|\nabla f_\theta(\hx_i )\|_2^2 R_i^{2}  c_x^2 \EE_\lambda [(1-\lambda)]^{2}d
\end{align*}
Now we bound $G_{1}= \frac{\EE_\lambda [(1-\lambda)]}{n}\sum_{i=1}^n (y_i-g(f_\theta(\hx_i )))f_\theta(\hx_i )$ by using  $\theta \in \Theta$. Since $\theta \in \Theta$, we have  $y_i f_{\theta}(\hx_i ) +(y_i-1)f_{\theta}(\hx_i )\ge 0$, which implies that $f_{\theta}(\hx_i )\ge 0$ if $y_i =1$ and $f_{\theta}(\hx_i )\le 0$ if $y_i=0$. Thus, if $y_i =1$,
$$
(y_i-g(f_{\theta}(\hx_i )))(f_{\theta}(\hx_i ))=(1-g(f_{\theta}(\hx_i )))( f_{\theta}(\hx_i )) \ge 0,
$$
since $(f_{\theta}(\hx_i )) \ge 0$ and $(1-g(f_{\theta}(\hx_i ))) \ge 0$ due to  $g(f_{\theta}(\hx_i )) \in (0,1)$.  If $y_i =0$,
$$
(y_i-g(f_{\theta}(\hx_i )))(f_{\theta}(\hx_i ))=-g(f_{\theta}(\hx_i ))(f_{\theta}(\hx_i )) \ge 0,
$$
since $(f_{\theta}(\hx_i ))\le 0$ and $-g(f_{\theta}(\hx_i )) <0$. Therefore, for all $i=1,\dots,n$, 
$$
(y_i-g(f_{\theta}(\hx_i )))(f_{\theta}(\hx_i )) \ge 0,
$$  
which implies that, since $\EE_\lambda [(1-\lambda)]\ge 0$,
\begin{align*}
G_{1}&= \frac{\EE_\lambda [(1-\lambda)]}{n}\sum_{i=1}^n |y_i-g(f_\theta(\hx_i ))||f_\theta(\hx_i )|
 \\ & = \frac{\EE_\lambda [(1-\lambda)]}{n}\sum_{i=1}^n |g(f_\theta(\hx_i ))-y_i| \|\nabla f_\theta(\hx_i )\|_2 \|\hx_i \|_2  |\cos(\nabla f_\theta(\hx_i ),  \hx_i )| 
 \\ & \ge \frac{1}{n}\sum_{i=1}^n |g(f_\theta(\hx_i ))-y_i| \|\nabla f_\theta(\hx_i )\|_2R_{i}c_x\EE_\lambda [(1-\lambda)] \sqrt{d}  
\end{align*}
By substituting these lower bounds of $G_{1}$ and $G_{2}$ into \eqref{eq:14}, we obtain 
\begin{align} 
\label{eq:15} & \Lcal_a - \frac{1}{n}\sum_{i=1}^n \ell(f_{\theta}(\hx_i ),y_{i})
\\ \nonumber &\ge  \frac{1}{n}\sum_{i=1}^n |g(f_{\theta}(\hx_i ))-y_i|\|\nabla f_\theta(\hx_i )\|_{2} \epsilon^{\mix}_i  + \frac{1}{2n}\sum_{i=1}^n |h''(f_\theta(\hx_i ))|\|\nabla f_\theta(\hx_i )\|_{2}^2 (\epsilon^{\mix}_i)^2  +E_{1}
\end{align}
On the other hand,    for any $z_1,\dots, z_n$, there exist  functions $\varphi'_i$ such that $\lim_{z \rightarrow 0}\varphi'_{i}(z)=0$,
and\begin{align}
\nonumber& \frac{1}{n} \sum_{i=1}^n \max_{\|\delta_{i}\|_2 \le \epsilon_{i} }\ell(f_{\theta}(z_{i}+\delta_{i}),y_{i})-\frac{1}{n} \sum_{i=1}^n \ell(f_{\theta}(z_{i}),y_{i}) 
\\ \nonumber &\le \scalebox{0.9}{$\displaystyle    \frac{1}{n} \sum_{i=1}^n|g(f_\theta(z_{i}))-y_i|\| \nabla f_\theta(z_{i}) \|_2 \epsilon_{i}  + \frac{1}{2n} \sum_{i=1}^{n}|h''(f_\theta(z_{i}))| \| \nabla f_\theta(z_{i})\|_2^2 \epsilon_i^{2} +\frac{1}{n} \sum_{i=1}^n\max_{\|\delta_{i}\|_2 \le \epsilon_{i}}\|\delta_{i}\|_2^{2}\varphi'_{i}(\delta_{i})$} 
\\ \label{eq:16} & \le    \frac{1}{n} \sum_{i=1}^n|g(f_\theta(z_{i}))-y_i|\| \nabla f_\theta(z_{i}) \|_2 \epsilon_{i}  + \frac{1}{2n} \sum_{i=1}^{n}|h''(f_\theta(z_i))| \| \nabla f_\theta(z_i)\|_2^2 \epsilon_i^{2}+\frac{1}{n} \sum_{i=1}^n\epsilon_{i}^2 \varphi''_{i}(  \epsilon_{i})
\end{align}
where $\varphi''_{i}(  \epsilon_{i})=\max_{\|\delta_{i}\|_2 \le \epsilon_{i} }\varphi'_{i}( \delta_{i})$. Note that  $\lim_{q \rightarrow 0}\varphi''_{i}(q)=0$. By combining \eqref{eq:15} and \eqref{eq:16}, 
$$
\Lcal_a \ge \frac{1}{n}\sum_{i=1}^n \max_{\substack{ \\ \|\delta^{\mix}_{i}\|_2 \le \epsilon^{\mix}_i }} \ell(f_{\theta}(x_{i}+\hdelta_{i}+\delta^{\mix}_{i}), y_i)+ \EE_\lambda [(1-\lambda )^2 \varphi_{}(1-\lambda)]-\frac{1}{n} \sum_{i=1}^n  (\epsilon^{\mix}_i)^2  \varphi''(\epsilon^{\mix}_i),
$$
where $\epsilon^{\mix}_i= R_{i}c_x\EE_\lambda[(1- \lambda)] \sqrt{d}$. Define $\bar \varphi(q)=-\frac{1}{n} \sum_{i=1}^n (R_{i}c_x \sqrt{d})^2 \varphi''_{i}(R_{i}c_x\sqrt{d}q)$. Note that   $\lim_{q \rightarrow 0} \bar \varphi(q)=0$. Then,
since $
\frac{1}{n} \sum_{i=1}^n (\epsilon^{\mix}_i)^2  \varphi''(\epsilon^{\mix}_i)
 =-E_{2},
$
\begin{align*}
\Lcal_a\ge  \frac{1}{n}\sum_{i=1}^n \max_{\substack{ \\ \|\delta^{\mix}_{i}\|_2 \le \epsilon^{\mix}_i }} \ell(f_{\theta}(x_{i}+\hdelta_{i}+\delta^{\mix}_{i}), y_i)+ E_{1}+E_{2}.
\end{align*}
where    $\lim_{z \rightarrow 0} \bar \varphi(z)=0$ and $\lim_{z \rightarrow 0} \varphi(z)=0$. 

\qed

\noindent \textbf{Proof of Theorem \ref{thm:4}}.  Let $S=((x_i, y_i))_{i=1}^n$ and $S'=((x_i', y_i'))_{i=1}^n$. Define 
\begin{align}
\varphi(S)= \sup_{f \in\Fcal} \EE_{x,y}[\ell(f(x),y)]-\frac{\Lcal_c +\Lcal_a}{2}.
\end{align} 
To apply McDiarmid's inequality to $\varphi(S)$, we compute an upper bound on $|\varphi(S)-\varphi(S')|$ where  $S$ and $S'$ be two test datasets differing by exactly one point of an arbitrary index $i_{0}$; i.e.,  $S_i= S'_i$ for all $i\neq i_{0}$ and $S_{i_{0}} \neq S'_{i_{0}}$. Then,
\begin{align}
\varphi(S')-\varphi(S) \le \frac{\tau (2n-1)}{n^2}\le \frac{2\tau}{n},
\end{align}
where we use the fact that both $\Lcal_c$ and $\Lcal_a$ have $n^2$ terms and $2n-1$ terms differ for $S$ and $S'$, each of which is bounded by the constant $\tau$. Similarly, $\varphi(S)-\varphi(S')\le \frac{2\tau}{n}$. Thus, by McDiarmid's inequality, for any $\delta>0$, with probability at least $1-\delta$,
\begin{align} \label{eq:12}
\varphi(S) \le  \EE_{S}[\varphi(S)] + 2\tau \sqrt{\frac{\ln(1/\delta)}{2n}}.
\end{align}
Moreover, by using Lemma \ref{lemma:1}, there exists functions $\varphi'$ and $\varphi''$ such that \begin{equation} 
\Lcal_a=\frac{1}{n}\sum_{i=1}^n \ell(f_{\theta}(\hx_i ),y_{i})+\sum_{i=1}^3 G_i+\bE_{\lambda\sim\tilde\cD_\lambda}[(1-\lambda)^2\varphi'(1-\lambda)],
\end{equation}
and 
\begin{equation} 
\Lcal_c=\frac{1}{n}\sum_{i=1}^n \ell(f_{\theta}(x_i ),y_{i})+\sum_{i=1}^3 R_i+\bE_{\lambda\sim\tilde\cD_\lambda}[(1-\lambda)^2\varphi''(1-\lambda)],
\end{equation}
where $\lim_{q\rightarrow 0}\varphi'(q)=0$ and  $\lim_{q\rightarrow 0}\varphi''(q)=0$. Thus, by defining
\begin{align}
Q(f)=\frac{1}{n}\EE_{S}\left[\sum_{i=1}^{n}\left(\max_{\delta_i:\|\delta_i\|_{\rho}\le \epsilon}\ell(f(x_i+\delta_i),y_i)-\ell(f(x_i),y_i)\right) \right],
\end{align}
and
\begin{align}
V(f)=\EE_{S}\left[\sum_{i=1}^3 \frac{G_i+R_i}{2}\right]-\bE_{\lambda\sim\tilde\cD_\lambda}[(1-\lambda)^2\varphi(1-\lambda)],
\end{align}
we have that
\begin{align}
&\EE_{S}[\varphi(S)] 
\\ &  = \EE_{S}\left[\sup_{f \in\Fcal} \EE_{S'}\left[\frac{1}{n}\sum_{i=1}^{n}\ell(f_{}(x_i'),y_i')\right]-\frac{\Lcal_c +\Lcal_a}{2} \right] 
\\ &  = \EE_{S}\left[\sup_{f \in\Fcal} \EE_{S'}\left[\frac{1}{n}\sum_{i=1}^{n}\ell(f_{}(x_i'),y_i')\right]-\frac{1}{n}\sum_{i=1}^{n}\ell(f(x_i),y_i) \right] - \frac{Q(f)}{2}-V(f)   
 \\ &  \le\EE_{S,S'}\left[\sup_{f \in\Fcal} \frac{1}{n}\sum_{i=1}^n (\ell(f_{}(x'_i),y'_i)-\ell(f(x_i),y_i)\right]  - \frac{Q(f)}{2}-V(f)
 \\ & \le \EE_{\xi, S, S'}\left[\sup_{f\in\Fcal} \frac{1}{n}\sum_{i=1}^n  \xi_i(\ell(f_{}(x'_{i}),y'_{i})-\ell(f(x_i),y_i))\right] - \frac{Q(f)}{2}-V(f)
 \\  &   \le2\EE_{\xi, S}\left[\sup_{f\in\Fcal} \frac{1}{n}\sum_{i=1}^n  \xi_i\ell(f(x_i),y_i))\right] - \frac{Q(f)}{2}-V(f)
 \\ \label{eq:13} & =2\Rcal_{n}(\ell \circ \Fcal)- \frac{Q(f)}{2}-V(f)
\end{align}
where  the second line follows the definitions of each term, the third line uses $\pm \frac{1}{n}\sum_{i=1}^{n}\ell(\allowbreak f(x_i),y_i)$ inside the expectation and the linearity of expectation, the forth line uses the Jensen's inequality and the convexity of  the 
supremum, and the fifth line follows that for each $\xi_i \in \{-1,+1\}$, the distribution of each term $\xi_i (\ell(f_{}(x'_i),y'_i)-\ell(f(x_i),y_i))$ is the  distribution of  $(\ell(f_{}(x'_i),y'_i)-\ell(f(x_i),y_i))$  since $S$ and $S'$ are drawn iid with the same distribution. The sixth line uses the subadditivity of supremum. 

Finally, by noticing that $Q(f)\ge 0$ from the definition of $Q(f)\ge 0$ (since $\max_{\delta_i:\|\delta_i\|_{\rho}\le \epsilon}\allowbreak \ell(f(x_i+\delta_i),y_i)-\ell(f(x_i),y_i)\ge 0$) and by combining equations \eqref{eq:12} and \eqref{eq:13}, we have the desired statement.

\qed

\noindent \textbf{Proof of Theorem \ref{thm:6}}.  Let $S=((x_i, y_i))_{i=1}^n$ and $S'=((x_i', y_i'))_{i=1}^n$. Define 
\begin{align}
\varphi(S)= \sup_{f \in\Fcal} \EE_{x,y}[\ell(f(x),y)]-\frac{\Lcal_c '+\Lcal_a'}{2}.
\end{align} 
To apply McDiarmid's inequality to $\varphi(S)$, we compute an upper bound on $|\varphi(S)-\varphi(S')|$ where  $S$ and $S'$ be two test datasets differing by exactly one point of an arbitrary index $i_{0}$; i.e.,  $S_i= S'_i$ for all $i\neq i_{0}$ and $S_{i_{0}} \neq S'_{i_{0}}$. Then,
\begin{align}
\varphi(S')-\varphi(S) \le \frac{2\tau}{n},
\end{align}
since $\sup_{f \in\Fcal}\frac{\max_{\delta_{i_0}:\|\delta_{i_0}\|_{\rho}\le \epsilon}\ell(f(x_{i_0}+\delta_{i_0}),y_{i_0})-\max_{\delta_{i_0}:\|\delta_{i_0}\|_{\rho}\le \epsilon}\ell(f_{}(x'_{i_0}+\delta_{i_0}),y'_{i_0})}{2n} \le\frac{\tau}{n} $. Similarly, $\varphi(S)-\varphi(S')\le \frac{2\tau}{n}$. Thus, by McDiarmid's inequality, for any $\delta>0$, with probability at least $1-\delta$,
\begin{align} \label{eq:2}
\varphi(S) \le  \EE_{S}[\varphi(S)] +2\tau \sqrt{\frac{\ln(1/\delta)}{2n}}.
\end{align}
Moreover, by defining
\begin{align}
Q(f)=\frac{1}{n}\EE_{S}\left[\sum_{i=1}^{n}\left(\max_{\delta_i:\|\delta_i\|_{\rho}\le \epsilon}\ell(f(x_i+\delta_i),y_i)-\ell(f(x_i),y_i)\right) \right],
\end{align}
we have that
\begin{align}
&\EE_{S}[\varphi(S)] 
\\ &  = \EE_{S}\left[\sup_{f \in\Fcal} \EE_{S'}\left[\frac{1}{n}\sum_{i=1}^{n}\ell(f_{}(x_i'),y_i')\right]-\frac{\Lcal_c '+\Lcal_a'}{2} \right] 
\\ &  = \EE_{S}\left[\sup_{f \in\Fcal} \EE_{S'}\left[\frac{1}{n}\sum_{i=1}^{n}\ell(f_{}(x_i'),y_i')\right]-\frac{1}{n}\sum_{i=1}^{n}\ell(f(x_i),y_i) \right] - \frac{Q(f)}{2}   
 \\ &  \le\EE_{S,S'}\left[\sup_{f \in\Fcal} \frac{1}{n}\sum_{i=1}^n (\ell(f_{}(x'_i),y'_i)-\ell(f(x_i),y_i)\right]  -  \frac{Q(f)}{2}
 \\ & \le \EE_{\xi, S, S'}\left[\sup_{f\in\Fcal} \frac{1}{n}\sum_{i=1}^n  \xi_i(\ell(f_{}(x'_{i}),y'_{i})-\ell(f(x_i),y_i))\right] - \frac{Q(f)}{2}
 \\ \label{eq:3} &   \le2\EE_{\xi, S}\left[\sup_{f\in\Fcal} \frac{1}{n}\sum_{i=1}^n  \xi_i\ell(f(x_i),y_i))\right] - \frac{Q(f)}{2}=2\Rcal_{n}(\ell \circ \Fcal)- \frac{Q(f)}{2}
\end{align}
where  the second line follows the definitions of each term, the third line uses $\pm \frac{1}{n}\sum_{i=1}^{n}\ell(\allowbreak f(x_i),y_i)$ inside the expectation and the linearity of expectation, the forth line uses the Jensen's inequality and the convexity of  the 
supremum, and the fifth line follows that for each $\xi_i \in \{-1,+1\}$, the distribution of each term $\xi_i (\ell(f_{}(x'_i),y'_i)-\ell(f(x_i),y_i))$ is the  distribution of  $(\ell(f_{}(x'_i),y'_i)-\ell(f(x_i),y_i))$  since $S$ and $S'$ are drawn iid with the same distribution. The sixth line uses the subadditivity of supremum. 

Finally, by noticing that $Q(f)\ge 0$ from the definition of $Q(f)\ge 0$ (since $\max_{\delta_i:\|\delta_i\|_{\rho}\le \epsilon}\allowbreak \ell(f(x_i+\delta_i),y_i)-\ell(f(x_i),y_i)\ge 0$) and by combining equations \eqref{eq:2} and \eqref{eq:3}, we have the desired statement.

\qed

\noindent \textbf{Proof of Theorem \ref{thm:5}}. From the assumption, we have $f_{\theta}(\hx_i )=\nabla f_\theta(\hx_i )\T \hx_i $ and $\nabla^2 f_\theta(\hx_i )=0$. Since $h(z)=\log(1+e^z)$, we have $h'(z)=\frac{e^z}{1+e^z}=g(z) \ge 0$ and $h''(z)=\frac{e^z}{(1+e^z)^{2}}=g(z)(1-g(z)) \ge 0$.
By substituting these into the equation of Lemma  \ref{lemma:1},
\begin{align}
\Lcal_a= \frac{1}{n}\sum_{i=1}^n \ell(f_{\theta}(\hx_i ),y_{i})+ G_{1}+ G_{2} +\EE_\lambda [(1-\lambda )^2 \varphi_{}(1-\lambda)],
\end{align} 
where 
$$
G_{1}= \scalebox{0.95}{$\displaystyle  \frac{\EE_\lambda [(1-\lambda)]}{n}\sum_{i=1}^n (y_i-g(f_\theta(\hx_i ))) \|\bE_{r\sim\cD_\hx}[r- \hx_i] \|_2 \cos(\nabla f_\theta(\hx_i),\bE_{r\sim\cD_\hx}[r- \hx_i]) \|\nabla f_\theta(\hx_i)\|_2, $} 
$$
\begin{align*}
G_{2}= \frac{\EE_\lambda [(1-\lambda)^{2}]}{2n} \sum_{i=1}^n |g(f_\theta(\hx_i ))(1-g(f_\theta(\hx_i )))| \|\nabla f_\theta(\hx_i )\|^2_{\EE_{r}[(r-\hx_i )(r-\hx_i )\T]}.
\end{align*}
\qed

\noindent \textbf{Proof of Proposition \ref{prop:1}}.
Since $\theta \in \Theta'$, we have  $y_i (f_{\theta}(\hx_i )-\zeta_{i} )+(y_i-1)(f_{\theta}(\hx_i )-\zeta_{i})\ge 0$, which implies that $f_{\theta}(\hx_i )-\zeta_{i}\ge 0$ if $y_i =1$ and $f_{\theta}(\hx_i )-\zeta_{i}\le 0$ if $y_i=0$. Thus, if $y_i =1$,
$$
(y_i-g(f_{\theta}(\hx_i )))(f_{\theta}(\hx_i )-\zeta_{i})=(1-g(f_{\theta}(\hx_i )))( f_{\theta}(\hx_i )-\zeta_{i}) \ge 0,
$$
since $(f_{\theta}(\hx_i )-\zeta_{i}) \ge 0$ and $(1-g(f_{\theta}(\hx_i ))) \ge 0$ due to  $g(f_{\theta}(\hx_i )) \in (0,1)$.  If $y_i =0$,
$$
(y_i-g(f_{\theta}(\hx_i )))(f_{\theta}(\hx_i )-\zeta_{i})=-g(f_{\theta}(\hx_i ))(f_{\theta}(\hx_i )-\zeta_{i}) \ge 0,
$$
since $(f_{\theta}(\hx_i )-\zeta_{i})\le 0$ and $-g(f_{\theta}(\hx_i )) <0$. Therefore, for all $i=1,\dots,n$, 
$$
(y_i-g(f_{\theta}(\hx_i )))(f_{\theta}(\hx_i )-\zeta_{i}) \ge 0.
$$  
This implies that, since $\EE_\lambda [(1-\lambda)]\ge 0$ and  $f_{\theta}(\hx_i )=\nabla f_\theta(\hx_i )\T \hx_i$,
\begin{align*}
G_{1}&= \frac{\EE_\lambda [(1-\lambda)]}{n}\sum_{i=1}^n (y_i-g(f_\theta(\hx_i )))(f_\theta(\hx_i )-\zeta_{i}) \ge 0.  
\end{align*}
Thus, we have that  $0 \le G_{1}=C_{1}\|\nabla f_\theta(\hx_i)\|_2$ where $\|\nabla f_\theta(\hx_i)\|_2 \ge 0$, which implies that $C_{1}\ge 0$.

\qed


\section*{References}

\bibliography{mybib}

\end{document}